\documentclass{article}

\usepackage{amsmath,amsfonts,bm}

\def\eqref#1{equation~\ref{#1}}

\def\1{\bm{1}}

\def\vmu{{\bm{\mu}}}

\def\vd{{\bm{d}}}

\def\vn{{\bm{n}}}

\def\vp{{\bm{p}}}

\def\vx{{\bm{x}}}

\def\mA{{\bm{A}}}

\DeclareMathAlphabet{\mathsfit}{\encodingdefault}{\sfdefault}{m}{sl}
\SetMathAlphabet{\mathsfit}{bold}{\encodingdefault}{\sfdefault}{bx}{n}

\DeclareMathOperator*{\argmin}{arg\,min}

\usepackage{microtype}
\usepackage{graphicx}
\usepackage{subcaption}
\usepackage{booktabs} 
\usepackage{hyperref}

\PassOptionsToPackage{numbers,sort&compress}{natbib}
\usepackage[preprint]{neurips_2026}

\usepackage{amsmath}
\usepackage{amssymb}
\usepackage{mathtools}
\usepackage{amsthm}
\usepackage{wrapfig}
\usepackage[dvipsnames]{xcolor}
\usepackage{url}
\usepackage{tikz}
\usepackage{pgfplots}
\usepackage{pgfplotstable}

\pgfplotsset{compat=1.18}
\usepgfplotslibrary{groupplots, statistics}

\definecolor{myblue}{RGB}{31,119,180}
\definecolor{myorange}{RGB}{255,127,14}

\pgfplotsset{
    discard if not x/.style 2 args={
        x filter/.append code={
            \edef\tempa{\thisrow{#1}}
            \edef\tempb{#2}
            \ifx\tempa\tempb
            \else
                
            \fi
        }
    }
}

\theoremstyle{plain}
\newtheorem{theorem}{Theorem}[section]
\newtheorem{proposition}[theorem]{Proposition}
\newtheorem{lemma}[theorem]{Lemma}

\theoremstyle{definition}
\newtheorem{definition}[theorem]{Definition}
\theoremstyle{remark}

\usepackage[capitalize,noabbrev]{cleveref}

\usepackage[textsize=tiny]{todonotes}

\title{The Linear Centroids Hypothesis: \\ Features as Directions Learned by Local Experts}

\author{%
  Thomas Walker\thanks{Correspondence to \texttt{thomas.walker@rice.edu}.} \\
  Rice University\\
  \And
  Ahmed Imtiaz Humayun \\
  Google Research \\
  \And
  Randall Balestriero \\
  Brown University \\
  \And
  Richard Baraniuk \\
  Rice University \\
}

\begin{document}

\maketitle

\begin{abstract}
The Linear Representation Hypothesis (LRH) identifies features of a trained deep network (DN) as linear directions in the activation spaces, i.e., output spaces of intermediate layers. 
This characterization decouples the input-output maps learned by a DN from the organization of feature directions in its activation spaces. 
We introduce the {\bf Linear Centroids Hypothesis (LCH)}, which instead identifies features with linear directions among a DN's \textit{centroid spaces} -- where any vector denotes a centroid or summary of a \textit{local affine expert} characterizing the learned input-output maps of the DN exactly (e.g., for piecewise-affine DNs) or approximately (e.g., for smooth DNs like transformers).
We show that replacing intermediate activations with centroids yields a functional drop-in alternative for standard interpretability tools.
Empirically, this change yields sparser, more downstream-useful feature dictionaries on DINO ViTs, suppresses spurious directions on a controlled task, recovers interpretable circuits in GPT2-Large, and produces faithful gradient-based saliency maps. 
LCH unifies dictionaries, probing, circuits, and saliency maps into a single geometric object grounded in the network's input-output map — making interpretability mechanistic by construction rather than post hoc.
\end{abstract}

\begin{figure}[!h]
     \centering
     \begin{subfigure}[b]{0.24\columnwidth}
         \centering
         \caption*{Input Sample}
         \includegraphics[width=\textwidth]{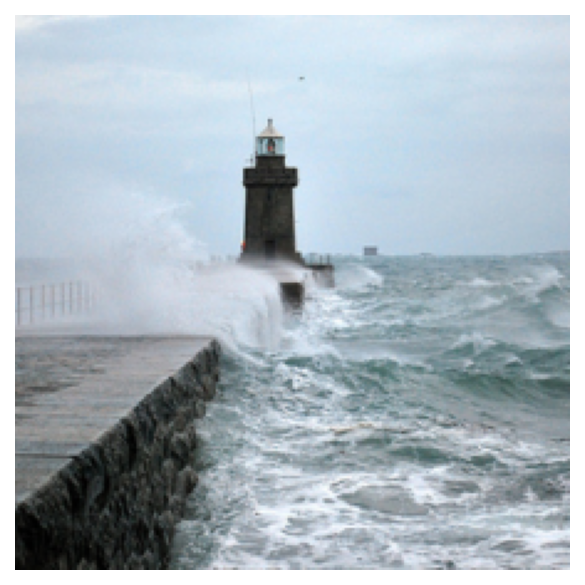}
     \end{subfigure}
     \begin{subfigure}[b]{0.24\columnwidth}
         \centering
         \caption*{Centroid}
         \includegraphics[width=\textwidth]{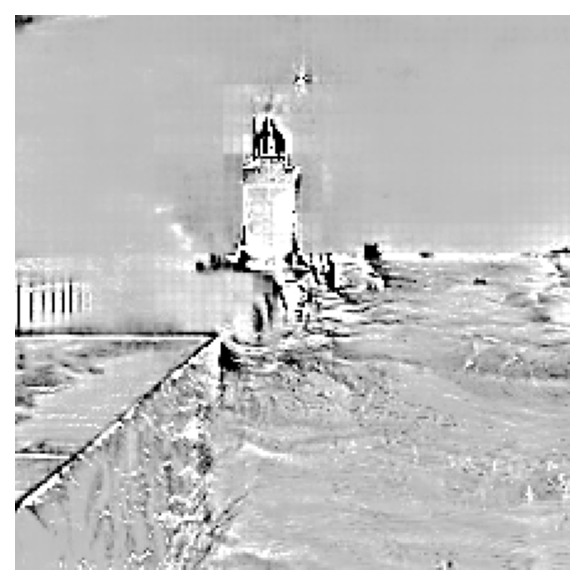}
     \end{subfigure}
     \hfill
     \begin{subfigure}[b]{0.24\columnwidth}
         \centering
         \caption*{Input Space}
         \includegraphics[width=\textwidth]{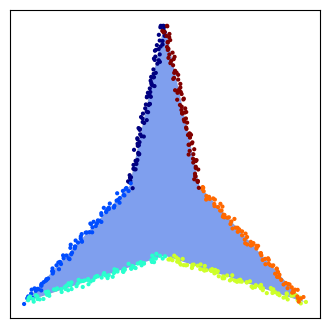}
     \end{subfigure}
     \begin{subfigure}[b]{0.24\columnwidth}
         \centering
         \caption*{Centroid Space}
         \includegraphics[width=\textwidth]{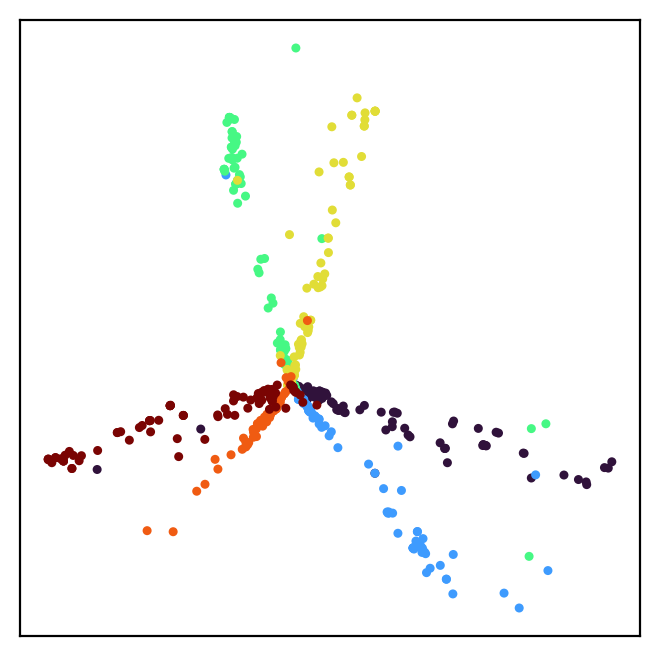}
     \end{subfigure}
    \caption{
    The \textbf{Linear Centroids Hypothesis} (LCH) posits that the {\em centroids} of a deep network (DN) represent features as linear directions.
    The centroids of a DN describe its functional mapping in a local region of the input space -- by being the row-sum of the DN's input-output Jacobians -- and thus can be considered descriptors of the local ``experts'' of a DN.
    In the first and second panels, we show that the centroids of a Swin-B~\citep{liuSwinTransformerHierarchical2021} transformer operating on inputs from ImageNet~\citep{krizhevskyImageNetClassificationDeep2012} can be used as saliency maps.
    In the third and fourth panels, we verify the LCH by training a DN to classify the interior from the exterior of a two-dimensional star-shaped polygon (third panel) and observing its centroids (fourth panel). 
    Indeed, the centroids operating on inputs sampled along the edge features (third panel) separate along linear directions (fourth panel).
    }
    \label{fig:1}
\end{figure}

\section{Introduction}\label{sec:introduction}

Understanding what \textit{features} a trained Deep Network (DN) has learned to extract from a given input distribution is a major focus of mechanistic interpretability research today ~\citep{amodeiConcreteProblemsAI2016,sharkeyOpenProblemsMechanistic2025}. \textit{Features} in the context of a DN are essentially a set of signals or patterns predictive of the output that a trained DN has learned to extract through its layer-wise operations. 
The current prevailing framework for discovery and identification of such features is the {\em Linear Representation Hypothesis} (LRH) that
posits features as linear directions in a DN's activation spaces, i.e., the output spaces of intermediate layers of a DN~\citep{elhage2022superposition,parkLinearRepresentationHypothesis2024}.
LRH has led to the development of various interpretability and feature discovery tools~\cite{kimInterpretabilityFeatureAttribution2018,trentonbrickenMonosemanticityDecomposingLanguage2023,hubenSparseAutoencodersFind2024}, operating on the assumption that any linear direction in which samples from an input distribution are ordered in an activation space, corresponds to a feature the network has learned to compute.

However, LRH suffers from three critical limitations.
First, it abstracts away from individual sub-components of the DN (e.g., neurons, layers, or attention mechanisms), making it difficult to link features to the computational graph~\citep{sharkeyOpenProblemsMechanistic2025}.
Second, it is unclear whether all linear directions of intermediate activations correspond to features, making it susceptible to identifying spurious features~\citep{researchNegativeResultsSparse2025}.
Third, its disjoint application to individual latent spaces has encumbered the field with the task of contextualizing features across sub-components~\citep{balaganskyMechanisticPermutabilityMatch2025}.
 
To make feature discovery mapping-aware, we introduce the \textbf{Linear Centroids Hypothesis (LCH)}.
LCH relies on the network's {\em centroids}, which directly summarize the local mapping or ``local experts'' learned by the network.
Surprisingly, we observe that samples with the same features naturally organize into linear directions within this centroid space, intuitively reframing features as ``aligned'' local experts (see third and fourth panels of \Cref{fig:1}).
LCH is motivated by the fact that continuous piecewise affine (CPA) DNs (e.g., ReLU networks) exactly and other smooth networks (e.g., transformers) approximately partition their input space into a Voronoi-like tiling of polytopal regions, on which distinct functional mappings (i.e., experts) operate~\citep{balestrieroGeometryDeepNetworks2019}.

Crucially, a DN's centroids can be efficiently calculated for any differentiable DN sub-component, since they are equal to the row-sum of its input-output Jacobian.
This transitions feature identification from understanding where in the space of intermediate representations an input lies to understanding the action of the local expert operating on the input, providing a grounded, mechanistic perspective on interpretability.

Our main contribution is providing a mechanistic framework for interpretability that is interchangeable with LRH by simply replacing intermediate activations with centroids.
We demonstrate that the transition to using the LCH prevents the identification of spurious features (see third panel of Figure \ref{fig:lch_evidence_vit}), improves the effectiveness of sparse autoencoders for constructing feature dictionaries (see Figure \ref{fig:centroid_saes}), facilitates the introduction of novel techniques for circuit discovery (see \Cref{fig:circuit_discovery}), and provides a faithful gradient-based saliency map in the form of the {\em local centroid} (see \Cref{fig:1,fig:saliency_map}).

Because LCH grounds feature discovery in the actual input-output map, it is significantly better at feature discovery and inherently resists spurious correlations. 
Ultimately, this centroid-based lens unifies feature dictionaries, probing, circuit discovery, and saliency maps under a single geometric object—the network’s induced input-space partition—making interpretability mechanistic by construction rather than post hoc.\footnote{Code to study the LCH can be found here: \url{https://github.com/ThomasWalker1/LinearCentroidsHypothesis}.} 

\section{Deep Networks as a Collection of Local Experts}\label{sec:background}

In this section, we demonstrate how DNs can be intuitively thought of as employing local experts to form their outputs, and that {\em centroids} describe these experts.

\subsection{Deep Networks}

A DN $f:\mathbb{R}^{d^{(0)}}\to\mathbb{R}^{d^{(L)}}$, with the convention that $d^{(0)}=d$, is a composition of $L$ functions $f^{(\ell)}:\mathbb{R}^{d^{(\ell-1)}}\to\mathbb{R}^{d^{(\ell)}}$.
These functions, referred to as layers, typically consist of an affine transformation followed by a nonlinearity.
We will denote the mapping across multiple layers as $f^{(\ell_1\leftarrow\ell_2)}:\mathbb{R}^{d^{(\ell_1-1)}}\to\mathbb{R}^{d^{(\ell_2)}}$ for $1\leq\ell_1<\ell_2\leq L$.

\subsection{The Geometry of Deep Networks}

For any DN,\footnote{Only requiring differentiability, which is the case almost everywhere for any DN.} the local behavior of the DN can be approximated locally with an affine transformation~\citep{lycheLocalSplineApproximation1975,schumakerSplineFunctionsBasic2007}.
Thus, for convenience, in the following, we consider continuous piecewise affine (CPA) DNs (i.e., those that exclusively use CPA operations, such as affine transformations, the ReLU activation function, or max pooling).
CPA DNs construct an implicit geometry by virtue of the fact that they partition their input space into local regions on which the input-output mapping corresponds to a specific affine transformation~\cite{balestrieroMadMaxAffine2018}.
We can think of these affine transformations as local ``experts.''


These regions can be parametrized as a {\em power-diagram subdivision} \citep{balestrieroGeometryDeepNetworks2019}, a hierarchical combination of {\em power diagrams}~\cite{imaiVoronoiDiagramLaguerre1985} (see Definition \ref{def:power_diagrams}).
Power diagrams are generalizations of the canonical Voronoi tiling, which tiles a space by identifying a finite set of {\em median} vectors and then assigning each point in the input space to the closest median with respect to the Euclidean distance.
In other words, each region (or tile) is the loci of the median vectors with respect to the Euclidean distance.
Power diagrams, are similar, except that regions are the loci of {\em centroids} with respect to Laguerre distance~\cite{imaiVoronoiDiagramLaguerre1985}.
The power diagram subdivision that parametrizes the geometry of a DN constitutes a layer-wise intersection of power diagrams~\cite{balestrieroGeometryDeepNetworks2019}.
This process is described in more detail in \Cref{sec:dn_geometry}.

Ultimately, each region of a DN's geometry has a {\em centroid}.
It is important to note that the centroid may not reside inside the region, indeed each region also has an associated {\em radius} parameter that influences the relative position of a centroid to the points in the region it defines.

\subsection{The Centroids of a Deep Network}

Since each sub-component of a CPA DN (e.g., layer or sequence of layers) is also CPA, each DN sub-component admits a geometry in its input space.
For $1\leq\ell_1<\ell_2\leq L$, we denote this geometry $\Omega^{(\ell_1\leftarrow\ell_2)}$ with regions $\left\{\omega^{(\ell_1\leftarrow\ell_2)}_{\boldsymbol{\nu}}\right\}_{\boldsymbol{\nu}\in\mathcal{V}}$.
Here, $\mathcal{V}$ denotes the set of equivalence classes on $\mathbb{R}^{d^{(\ell_1-1)}}$ of regions.
For notational simplicity, we will use $\boldsymbol{\nu}(\vx)\in\mathcal{V}$ to denote the region occupied by $f^{(1\leftarrow\ell_1-1)}(\vx)\in\mathbb{R}^{d^{(\ell_1-1)}}$.
Consequently, for a given point $\vx\in\mathbb{R}^d$, we can identify its corresponding centroid at a sub-component as $\vmu_{\boldsymbol{\nu}(\vx)}^{(\ell_1\leftarrow\ell_2)}\in\mathbb{R}^{d^{(\ell_1-1)}}$.

\begin{proposition}[\citealt{balestrieroGeometryDeepNetworks2019}]\label{prop:centroid_jacobian}
    Let $\boldsymbol{J}_{\vx}\left(f^{(\ell_1\leftarrow\ell_2)}\right)\in\mathbb{R}^{d^{(\ell_2)}\times d^{(\ell_1-1)}}$ denote the input-output Jacobian of $f^{(\ell_1\leftarrow\ell_2)}$ at $f^{(1\leftarrow\ell_1-1)}(\vx)$. Then, $\vmu^{(\ell_1\leftarrow\ell_2)}_{\boldsymbol{\nu}(\vx)}=\left(\boldsymbol{J}_{\vx}\left(f^{(\ell_1\leftarrow\ell_2)}\right)\right)^\top\mathbf{1}$.
\end{proposition}

Note how Proposition \ref{prop:centroid_jacobian} enables the efficient computation of DN centroids at individual points in the input space.
Moreover, observe that for $\ell_2$ distinct from $\ell_2^\prime$, the centroids $\vmu_{\boldsymbol{\nu}}^{(\ell_1\leftarrow\ell_2)}$ and $\vmu_{\boldsymbol{\nu}}^{(\ell_1\leftarrow\ell^\prime_2)}$ live in the same space $\mathbb{R}^{d^{(\ell_1-1)}}$, but they are distinct since the maps $f^{(\ell_1\leftarrow\ell_2)}$ and $f^{(\ell_1\leftarrow\ell^\prime_2)}$ are distinct.
This motivates the \emph{centroid space} terminology, to refer to the space of centroids corresponding to one specific sub-component of the DN.

Importantly, Proposition \ref{prop:centroid_jacobian} extends this geometrical perspective beyond CPA DNs, as Jacobian vector products are readily computable for any differentiable mapping.

\subsection{Centroids as Saliency Maps}

Centroids describe the action and arrangement of the local experts of a DN, and reside in the input space of the sub-component on which they are computed.
Meaning, for sub-components whose input space is that of the DN, centroids have a visceral resemblance to a saliency map~\cite{simonyanDeepConvolutionalNetworks2014,smilkovSmoothGradRemovingNoise2017,selvarajuGradCAMVisualExplanations2017,sundararajan2017axiomatic}.
Indeed, with the first two panels of Figure \ref{fig:1}, we can see that the centroid of an input sample highlights its key features, like edges.

With Figure \ref{fig:saliency_map}, we support this by showing that the centroid (second panel) of a standard pre-trained ConvNext-L DN~\cite{liuConvNet2020s2022} at an input point (first panel) resembles some of the characteristics of the input.
In particular, we can improve the {\em signal} by consulting a neighborhood of local experts, namely, averaging the centroids of samples in a small neighborhood of the input point.\footnote{This is analogous to SmoothGrad~\cite{smilkovSmoothGradRemovingNoise2017}.}
Doing so yields a saliency map that vividly highlights characteristics of the input (third panel).

To reinforce the premise that these highlighted characteristics are those pertinent to the output of the DN, we repeat the computation but for an adversarially trained ConvNext-L DN~\cite{liuComprehensiveStudyRobustness2024}.
In the fourth panel of \Cref{fig:saliency_map}, we observe that the identified characteristics are much more salient, as expected, because the DN exhibits significantly greater robustness.
Consequently, we introduce the {\em local centroid} of an input, namely the average centroid obtained from neighborhood samples of the input, as a saliency method.

In \Cref{sec:local_centroids}, we explore the local centroids on other inputs and DN architectures.
Specifically, since it has been shown that current gradient-based saliency methods provide essentially unchanged explanations for both trained and untrained DNs~\citep{adebayo2018sanity}, we demonstrate that local centroids of randomly initialized DNs are uninformative.
Furthermore, we show how taking the local centroids of DN sub-components from the input space to an intermediate layer can elucidate the hierarchical features of a DN.

\begin{figure}[ht]
     \centering
     \begin{subfigure}[b]{0.24\columnwidth}
         \centering
         \includegraphics[width=\textwidth]{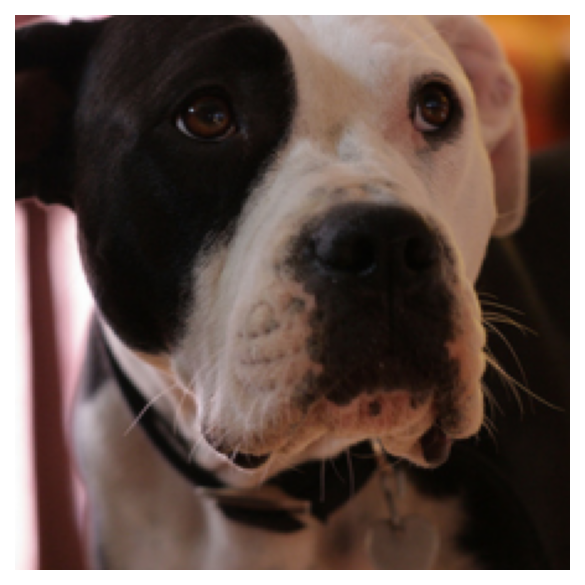}
         \caption*{\centering Input\\Sample}
     \end{subfigure}
     \hfill
     \begin{subfigure}[b]{0.24\columnwidth}
         \centering
         \includegraphics[width=\textwidth]{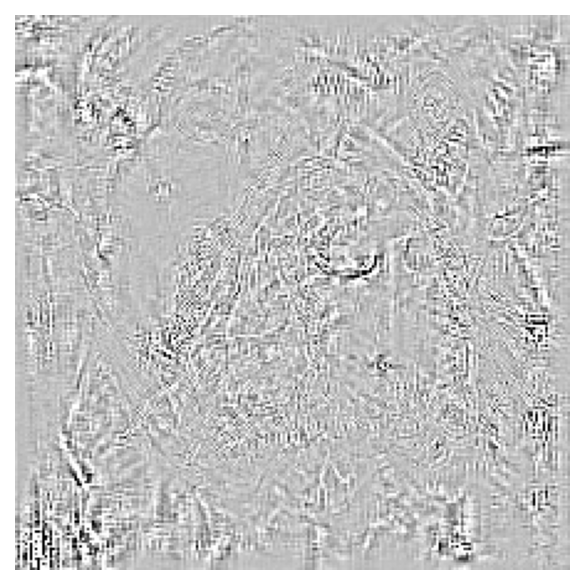}
         \caption*{\centering Centroid\\(Standard Training)}
     \end{subfigure}
     \begin{subfigure}[b]{0.24\columnwidth}
        \centering
        \includegraphics[width=\textwidth]{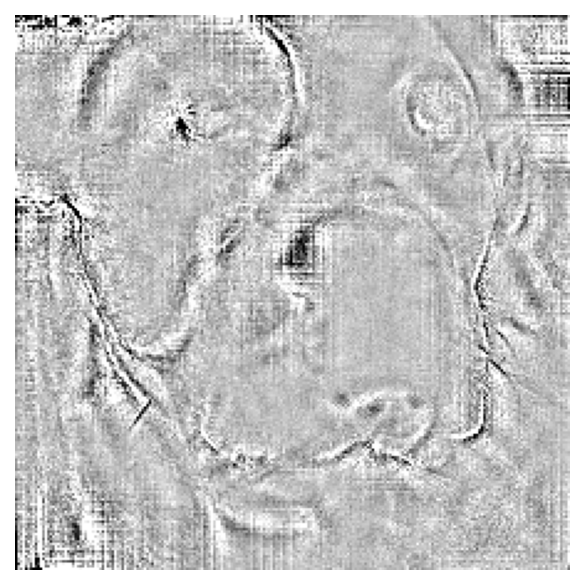}
        \caption*{\centering Local Centroid\\(Standard Training)}
     \end{subfigure}
     \hfill
     \begin{subfigure}[b]{0.24\columnwidth}
        \centering
        \includegraphics[width=\textwidth]{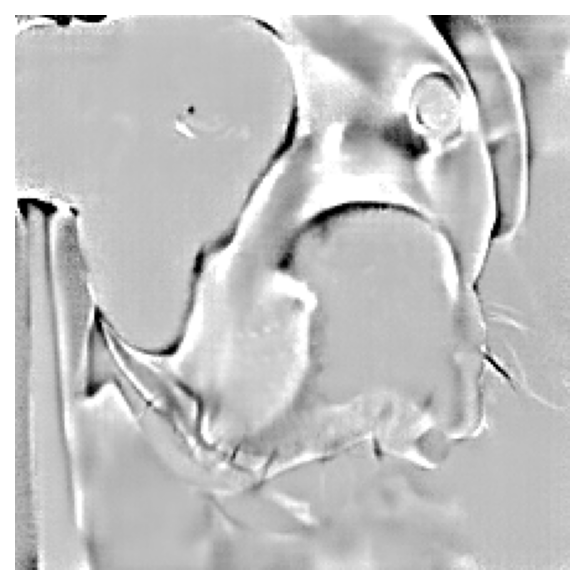}
        \caption*{\centering Local Centroid\\(Adversarial Training)}
     \end{subfigure}
    \caption{
    The centroids of a DN can be used as a saliency map.
    Here, we compute the centroids of a ConvNext-L DN either with the pre-trained weights from PyTorch~\cite{torchvision2016} or with weights obtained using the adversarial training methods of \citet{liuComprehensiveStudyRobustness2024}, at an input (first panel).
    In the second and third panels, we visualize the DN's centroid and local centroid, respectively, trained using standard methods.
    In the fourth panel, we visualize the local centroid of the adversarially trained DN.
    }
    \label{fig:saliency_map}
\end{figure}

\section{The Linear Centroids Hypothesis}\label{sec:centroid_affinity_identifies_features}

In \Cref{sec:features_and_circuits}, we make explicit what we mean by the {\em features of a DN} and highlight that the Linear Representation Hypothesis (LRH) is susceptible to identifying {\em spurious} features.
In \Cref{sec:deriving_CAH}, we propose the Linear Centroids Hypothesis (LCH) as a way to overcome this limitation of the LRH, since centroids are inherently map-aware.
In Section \ref{sec:spurious_features}, we confirm that the LCH is not susceptible to identifying spurious features.

\subsection{Defining Features and Circuits}\label{sec:features_and_circuits}

Inputs to DNs possess defining characteristics (e.g., the characteristic of having whiskers may be present for input images of cats).
A DN effectively groups its input space into regions corresponding to combinations of these characteristics, which it can then extract to compute its final output. 

The {\em features} of a DN refer to the specific characteristics the DN actively {\em utilizes} to form its output (e.g., using the characteristic of having whiskers to classify the input as a cat).
Crucially, utilization requires a two-step process: {\em extracting} a common representation for inputs sharing the characteristic, and ensuring this representation meaningfully {\em influences} the network's downstream behavior.
We refer to the influence induced by a particular feature as the corresponding circuit.
We relate this notion of DN features to prior works in \Cref{sec:comparison_to_literature}.

It is entirely possible for a DN to extract a representation for an input characteristic without it ever influencing the final output.
We refer to such characteristics as \emph{spurious} features.
The task of interpretability is to identify the features and circuits of a DN that are not spurious~\citep{sharkeyOpenProblemsMechanistic2025}.

Currently, the prevailing framework for interpretability is the Linear Representation Hypothesis (LRH), which posits that the features of a DN can be identified with linear directions formed by the intermediate activations of inputs~\citep{elhage2022superposition,parkLinearRepresentationHypothesis2024}.
This approach is limited, as analyzing an intermediate representation in isolation ignores the downstream mapping.

For ReLU DNs, recall that the zero-level sets of the neurons construct the DN's geometry.
Thus, the influence of an input is determined by the sign of the intermediate activations, which in turn determines which local expert the DN employs for that input. 
Using this, we can demonstrate that representing input characteristics as linear directions of intermediate activations is only a {\em necessary} condition for ensuring they influence the behavior of the DN.

\begin{lemma}\label{lem:necessary_cond}
    Let $\mathcal{X}_1$ and $\mathcal{X}_2$ be two sets of inputs with distinct characteristics. 
    A necessary condition, but not sufficient condition, for a ReLU DN to represent $\mathcal{X}_1$ and $\mathcal{X}_2$ as distinct features, is that the corresponding intermediate activations at some layer form linear directions.
\end{lemma}

\textit{Proof.} For necessity, assume the network successfully represents $\mathcal{X}_1$ and $\mathcal{X}_2$ as distinct features. 
For the representations to have consistently different influences, the network must assign different activation patterns to the inputs of $\mathcal{X}_1$ and $\mathcal{X}_2$. 
This means that, at some intermediate layer, at least one neuron activates exclusively for the latent activations of only $\mathcal{X}_1$ or $\mathcal{X}_2$. 
Therefore, the activation level set of this neuron must separate the latent activations of the two input sets. 
Because the activation level-set of a ReLU neuron is a hyperplane, the latent activations of $\mathcal{X}_1$ and $\mathcal{X}_2$ must lie in distinct half-spaces defined by this hyperplane. 
Geometrically, being strictly separable by a hyperplane implies that the activations separate along distinct linear directions in that intermediate feature space.

For insufficiency, assume that the intermediate activations of $\mathcal{X}_1$ and $\mathcal{X}_2$ do form distinct linear directions (i.e., are linearly separable) at some intermediate layer $l$. 
This intermediate separation does not guarantee that the final network output will represent them as distinct features. 
A subsequent layer $l+k$ could easily map these separated activations to the exact same output representation. 
For example, if the weight matrix of a subsequent layer is composed entirely of zeros, or if a sufficiently large negative bias is applied such that the ReLU function outputs zero for all inputs from both $\mathcal{X}_1$ and $\mathcal{X}_2$, the previously distinct representations collapse into a single point.
Therefore, while intermediate linear separation is required, it is not sufficient to guarantee distinct feature representation at the output.\qed

Lemma \ref{lem:necessary_cond} formalizes a flaw in the LRH, in that it assumes features are mapped along linear subspaces but entirely decouples this from the learned input-output map.
Specifically, an intermediate linear separation can be easily collapsed by subsequent layers; thus, there remains an ambiguity as to whether identified features actually drive the network's mapping or are merely spurious artifacts~\citep{smithStrongFeatureHypothesis2024}.
This mapping-blindness may explain why current interpretability techniques that study the structure of intermediate activations, such as linear classifiers~\cite{kimInterpretabilityFeatureAttribution2018} and sparse autoencoders~\cite{trentonbrickenMonosemanticityDecomposingLanguage2023,hubenSparseAutoencodersFind2024}, exhibit brittleness when applied to downstream tasks~\cite{nicolsonExplainingExplainabilityRecommendations2025,researchNegativeResultsSparse2025}, and are insufficient for explaining the functional behavior of DN sub-components~\citep{sharkeyOpenProblemsMechanistic2025}.

\subsection{Using Geometry to Identify Features}\label{sec:deriving_CAH}

To overcome the LRH's mapping-blindness, we propose anchoring feature discovery to the network's actual input-output mapping via its centroids.
In particular we propose the \textbf{Linear Centroids Hypothesis (LCH)}:
\begin{center}
    \textit{The features of a DN are represented by linear directions in its centroid spaces.}
\end{center}
Intuitively, the LCH states that the features of a DN can be identified by observing the collection of inputs acted on by ``aligned'' local experts.
Because centroids are derived directly from the network's Jacobians, this formulation is inherently mapping-aware.

In Appendix \ref{sec:formal_theory}, we understand what the LCH means from the perspective that views the geometry of a DN as constructed through the collection of hyperplanes defined by each neuron of the DN (see Appendix \ref{sec:dn_geometry}).
Although this is only applicable to CPA DN, in this context, linearly aligned centroids are equivalent to a partition geometry formed by hyperplane boundaries that accumulate along boundaries.
Consequently, the LCH is supported by prior works that empirically observe that this type of structured geometry explains many properties of DNs.
\citet{humayunDeepNetworksAlways2024} demonstrates that the generalization and robustness properties of DNs -- including transformers, residual networks, and convolutional neural networks -- emerge as the regions of the DN's geometry accumulate along the decision boundaries of the input space.
Similarly, the size of the linear region was linked to the toxicity property of large language models \citep{balestrieroCharacterizingLargeLanguage2023}.
Likewise, the density of linear regions in the geometry of DN sub-components has been connected to the reasoning capabilities of large language models \citep{cosentinoReasoningLargeLanguage2024}.

\begin{figure}[t]
     \centering
     \begin{subfigure}[b]{0.28\columnwidth}
         \centering
         \includegraphics[width=\textwidth]{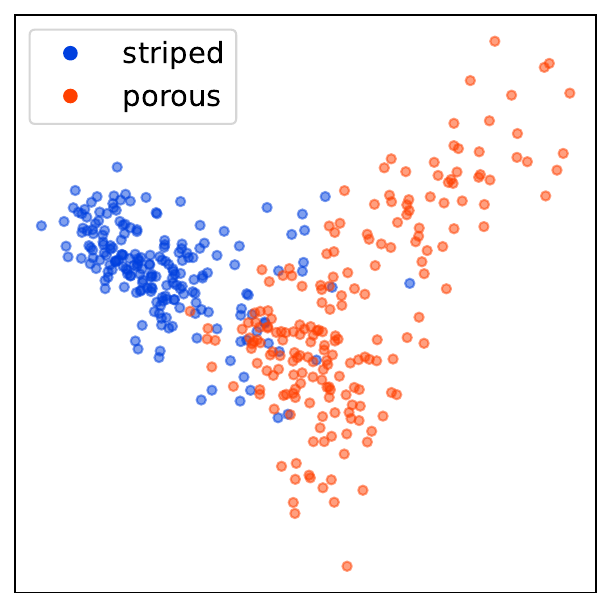}
     \end{subfigure}
     \begin{subfigure}[b]{0.28\columnwidth}
         \centering
         \includegraphics[width=\textwidth]{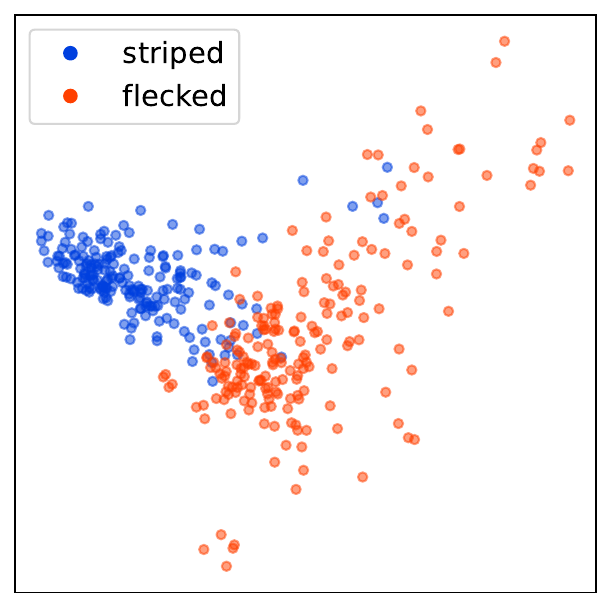}
     \end{subfigure}
    \begin{subfigure}[b]{0.34\textwidth}
        \centering
        \begin{tikzpicture}
            \begin{axis}[
                width=5.5cm,
                height=4.5cm,
                xlabel={Correlation},
                ylabel={Accuracy},
                legend pos=south east,
                legend style={font=\scriptsize},
                label style={font=\footnotesize},
                tick label style={font=\footnotesize}
            ]
            
            \addplot[
                color=myorange, thick, mark=*, mark size=1pt,
                error bars/.cd, y dir=both, y explicit
            ] table [
                x=correlation, 
                y=latent_accuracy_mean, 
                y error=latent_accuracy_std, 
                col sep=comma
            ] {data/spurious_feature_hypothesis.csv};
            \addlegendentry{LRH}
            
            \addplot[
                color=myblue, thick, mark=square*, mark size=1pt,
                error bars/.cd, y dir=both, y explicit
            ] table [
                x=correlation, 
                y=gradient_accuracy_mean, 
                y error=gradient_accuracy_std, 
                col sep=comma
            ] {data/spurious_feature_hypothesis.csv};
            \addlegendentry{LCH}
            
            \end{axis}
        \end{tikzpicture}
    \end{subfigure}
    \caption{
    \textbf{The LCH holds true for pre-trained vision models on ImageNet and mitigates the identification of spurious features.}
    In the first and second panels, we show that the centroids of inputs representing distinct features of the DN separate along linear directions under principal component analysis (PCA).
    We study the centroids of the fourth layer of a pre-trained ResNet50 model on ImageNet.
    We perform PCA on the centroids of inputs from two distinct classes in the DTD dataset.
    In the third panel, we train a DN on a version of FashionMNIST that has a color feature correlated with the corresponding classification task.
    We train DNs on this dataset at different correlation levels and measure the ability of a linear probe to extract the color feature using either the intermediate activations (LRH) or the centroids (LCH) of the DN.
    We consider five randomly initialized DNs at each correlation, and report the mean and standard deviation of the probe accuracy.
    }
    \label{fig:lch_evidence_vit}
\end{figure}

We provide evidence for the LCH with the third and fourth panels of Figure \ref{fig:1} and the first and second panels of Figure \ref{fig:lch_evidence_vit}.

In the third and fourth panels of \ref{fig:1}, we consider a DN trained to classify whether two-dimensional input points are inside or outside the star-shaped polygon shown in the third panel of Figure \ref{fig:1}.
In this instance, the characteristics of the input space that the DN ought to acquire as features are the interior and exterior of the polygon.
By sampling input points from the edges of the star within the DN's input space and observing their centroids, we can clearly see the emergence of linear directions, see the fourth panel of Figure \ref{fig:1}.
In \Cref{sec:other_polygons}, we replicate this for other polygons and DNs that are not CPA.

Similarly, for a ResNet50 DN~\citep{heDeepResidualLearning2016} pre-trained on ImageNet~\citep{krizhevskyImageNetClassificationDeep2012}, we observe in the first and second panels of Figure \ref{fig:lch_evidence_vit} that the centroids for inputs exhibiting distinct texture-based features (obtained using the DTD dataset~\citep{cimpoi14describing}) separate into distinct linear directions under a principal component analysis.

In \Cref{sec:prh}, we further support the LCH by showing that it validates the Platonic Representation Hypothesis by demonstrating that increasingly larger models--of different modalities--converge to the same representations of features as characterized by linear directions of centroids.

\subsection{The Linear Centroids Hypothesis Mitigates Spurious Feature Identification}\label{sec:spurious_features}

As established by Lemma \ref{lem:necessary_cond}, the LRH suffers from an ambiguity as to whether intermediate activations that form linear directions actually correspond to features or are merely spurious~\citep{smithStrongFeatureHypothesis2024}. 
In contrast, the LCH is mapping-aware, as it identifies structures in the centroids that are inherently tied to the functional behavior of the DN through the Jacobian. 
We demonstrate that this grounding in the input-output map makes LCH less susceptible to spurious features than the LRH.

For this, we train a convolutional neural network to classify the FashionMNIST dataset~\citep{xiaoFashionMNISTNovelImage2017}.
However, at initialization, we color the images by sampling from a discrete set of 10 colors and varying the degree of correlation between the coloring and the dataset labels.
When the coloring is fully correlated, the color feature is not spurious; when it is random, it is spurious.
We evaluate the extent to which the color feature is linearly represented by training a linear probe to predict color from the DN's intermediate activations or centroids and measuring its accuracy.
In the third panel of \Cref{fig:lch_evidence_vit}, it is clear that the centroids of the color feature form linear directions to the extent that the feature is relevant to the task, in contrast to intermediate activations, which are linear even when the feature is spurious.
This demonstrates that LCH effectively filters out spurious correlations.

\section{Interpretability Under the Linear Centroids Hypothesis}\label{sec:experiments}

We now replicate standard LRH-based interpretability studies under the LCH, on models including DINOv2~\citep{oquabDINOv2LearningRobust2024}, DINOv3~\citep{simeoniDINOv32025}, GPT2~\citep{radfordLanguageModelsAre2019}, and Llama-3.1-8B~\citep{grattafioriLlama3Herd2024}.
In \Cref{sec:compute_requirements}, we detail the differences in computational resources required to perform these experiments under the LCH compared to under the LRH.
The only difference arises in the extraction of the activations, for which centroids take around $10$-$15\%$ longer.
In the grand scheme of running these models, this addition is negligible.

\subsection{Feature Extraction with Sparse Autoencoders}\label{sec:saes}

Here, we compare sparse autoencoders trained on the intermediate activations and centroids from DINOv2 and DINOv3.
Because centroids encode the local mapping rather than just spatial positioning in a latent space, we hypothesize that feature dictionaries built on LCH will be more semantically robust and less prone to spurious correlations.
Similar to \citet{hindupurProjectingAssumptionsDuality2025}, we extract the intermediate activations and centroids of Imagenette \citep{howardFastai2020} from these models at the last multi-layer perceptron block to train a TopK sparse autoencoder \citep{gaoScalingEvaluatingSparse2025}.

We compare the obtained DINOv2 feature dictionaries in the following ways:
Train linear probes on the feature decompositions of the train set of Imagenette to classify its classes, and then evaluate the accuracy of the probe on the feature decompositions of the test set of Imagenette. Record the frequency at which the features of the sparse autoencoder fire on the test set of Imagenette.

\pgfplotstableread[col sep=comma]{data/cls_comparison_10exp_32K.csv}\datatableCLS

\pgfplotstablecreatecol[
    create col/assign/.code={
        \edef\tempa{\thisrow{activation_type}}
        \edef\tempb{centroids}
        \ifx\tempa\tempb
            \edef\tempval{\thisrow{cosine_similarity}}
            \pgfkeyslet{/pgfplots/table/create col/next content}\tempval
        \else
            \def\tempNaN{NaN}
            \pgfkeyslet{/pgfplots/table/create col/next content}\tempNaN
        \fi
    }
]{centroids_sim}{\datatableCLS}

\pgfplotstablecreatecol[
    create col/assign/.code={
        \edef\tempa{\thisrow{activation_type}}
        \edef\tempb{latents}
        \ifx\tempa\tempb
            \edef\tempval{\thisrow{cosine_similarity}}
            \pgfkeyslet{/pgfplots/table/create col/next content}\tempval
        \else
            \def\tempNaN{NaN}
            \pgfkeyslet{/pgfplots/table/create col/next content}\tempNaN
        \fi
    }
]{latents_sim}{\datatableCLS}

\begin{figure}[ht]
    \centering
    \begin{tikzpicture}
    \begin{groupplot}[
        group style={
            group size=3 by 1,
            horizontal sep=1.2cm,
        },
        width=5cm,
        height=4.5cm,
        grid=major,
        grid style={dashed, gray!30},
        legend style={font=\scriptsize, cells={anchor=west}},
        tick label style={font=\scriptsize},
        label style={font=\scriptsize},
        title style={font=\bfseries}
    ]

    \nextgroupplot[
        xlabel={K},
        ylabel={Probe Accuracy},
        xmode=log,
        log basis x=2,
        xtick={8,16,32,64},
        xticklabels={8,16,32,64},
        ymin=88, ymax=96,
    ]
    \addplot[color=myblue, mark=*, thick] 
        table[x=K, y=Centroids, col sep=comma] {data/probe_accuracy_vs_sparsity.csv};

    \addplot[color=myorange, mark=square*, thick, dashed] 
        table[x=K, y=Latents, col sep=comma] {data/probe_accuracy_vs_sparsity.csv};

    \nextgroupplot[
        xlabel={Rank},
        ylabel={Log Activation Frequency},
        ymin=0, ymax=6,
        legend pos=north east,
        each nth point={20},
        filter discard warning=false
    ]
    \addplot[color=myblue, thick] 
        table[x=rank, y=log_activation_frequency, col sep=comma, discard if not x={series}{Centroids}] {data/all-exp10-k32-seed0_firing_dist.csv};

    \addplot[color=myorange, thick, dashed] 
        table[x=rank, y=log_activation_frequency, col sep=comma, discard if not x={series}{Latents}] {data/all-exp10-k32-seed0_firing_dist.csv};

    \nextgroupplot[
        xlabel={Cosine Similarity},
        ylabel={Density},
        ymin=0,
        enlarge x limits=false,
        legend pos=north west
    ]
    \addplot [
        hist={density, bins=40, data min=0, data max=1},
        fill=myblue, fill opacity=0.4, draw=myblue, thick
    ] table [y=centroids_sim] {\datatableCLS};
    \addlegendentry{LCH}

    \addplot [
        hist={density, bins=40, data min=0, data max=1},
        fill=myorange, fill opacity=0.4, draw=myorange, thick
    ] table [y=latents_sim] {\datatableCLS};
    \addlegendentry{LRH}

    \end{groupplot}
    \end{tikzpicture}
    
    \caption{
    \textbf{Feature dictionaries from sparse autoencoders trained on centroids transfer better on downstream tasks, are more active on unseen inputs, and persist more coherently across model sizes.}
    In the first panel, we report the accuracy of linear probes on the Imagenette test set, trained to classify an input's label based on which features in the feature dictionary it activates.
    In the second panel, we measure the frequency with which features fire when activations from the Imagenette test set are passed through the sparse autoencoder.
    In the third panel, we measure the maximum cosine similarity between a feature from the DINOv2 dictionary and features from the DINOv3 dictionary.}
    \label{fig:centroid_saes}
\end{figure}

With the first panel of Figure \ref{fig:centroid_saes}, we observe that the LCH feature dictionary exhibits greater generalization, as compared to the LRH feature dictionary.
We reproduce this result for sparse autoencoders trained on a ten-class subset of ImageNet~\cite{krizhevskyImageNetClassificationDeep2012} containing dog breeds in \Cref{sec:qualitative}.
In particular, we use this setting to also demonstrate that qualitatively, the features identified using LCH are more semantically coherent.

With the second panel of Figure \ref{fig:centroid_saes}, we show that more of the LCH feature dictionary is more active on test samples from Imagenette than the LRH feature dictionary.
This highlights how, under the LCH, identified features are less likely to be spurious.

Next, we compare the feature dictionaries obtained from DINOv2 and DINOv3.
For a given activation type (i.e., intermediate activations or centroids), we compute the maximum cosine similarity of each DINOv2 feature with the full DINOv3 dictionary.
Since DINOv3 refines the features learned by DINOv2, the dictionary should exhibit high cosine similarities with those features.
Indeed, this is what we observe in the third panel of Figure \ref{fig:centroid_saes} for the feature dictionaries obtained using centroids.
In contrast, for the dictionary obtained using intermediate activations, a large proportion of the features have cosine similarities of around $0.4$.
This bimodal distribution suggests that the features identified by LRH do not correlate across the models and are instead spurious artifacts. 
In contrast, because centroids capture the functional mapping rather than arbitrary intermediate geometries, they converge reliably across model scales, strongly validating the Platonic Representation Hypothesis~\citep{huh2024position}.

\subsection{Circuit Discovery using Attribution Metrics}

Because centroids explicitly summarize the local experts driving the input-output map, their direct relationship with the computational graph offers a powerful mechanism for circuit discovery.
Rather than relying on activation magnitudes, we can directly query the mapping to introduce a novel attribution metric using centroids~\citep{mengLocatingEditingFactual2022,wangInterpretabilityWildCircuit2023,goldowsky-dillLocalizingModelBehavior2023}, and demonstrate how it can be used to perform circuit discovery on GPT2-Large~\cite{radfordLanguageModelsAre2019}.

Formally, let $f$ be a DN and $f^{(i,\ell)}$ be the same DN but with the $i^\text{th}$ neuron of the $\ell^{\text{th}}$ layer manipulated. 
The attribution of neuron $i$ to the features of a collection of samples $\mathcal{N}$ is quantified as 
\begin{equation}\label{eq:neuron_attribution}
    s_{\mathcal{N}}^{(i,\ell)}:=\frac{1}{\vert\mathcal{N}\vert}\sum_{\vx\in\mathcal{N}}\frac{\left\Vert\vmu^{f}_{\vx}-\vmu^{f^{(i,\ell)}}_{\vx}\right\Vert_2}{\left\Vert\vmu^{f}_{\vx}\right\Vert_2},
\end{equation}
where $\vmu_{\vx}^f$ and $\vmu^{f^{(i,\ell)}}_{\vx}$ are the centroids of $f$ and $f^{(i,\ell)}$ at $\vx$ respectively.
Quantifying the attribution of a neuron to the local features of a sample point $\vx$ can be done by taking $\mathcal{N}$ to be $\mathcal{B}_{\epsilon}(\vx)=\left\{\vx^\prime\in\mathbb{R}^d:\left\Vert\vx-\vx^\prime\right\Vert_2<\epsilon\right\}$.\footnote{Henceforth, we will use $s^{(i,\ell)}$ to denote $s_{\mathcal{B}_{\epsilon}(\vx)}^{(i,\ell)}$ unless stated otherwise.}

By patching each neuron in a layer of GPT2-Large, \citet{clementWeFoundNeuron2023} observed that the thirty-first layer multi-layer perceptron contains a neuron which is responsible for predicting the ``an'' token. 

We demonstrate that similar analyses can be conducted more simply by applying \Cref{eq:neuron_attribution} to a neighborhood of the last token embeddings at the input of the thirty-first layer on the prompt ``I climbed up the pear tree and picked a pear. I climbed up the apple tree and picked.'' 
In \Cref{fig:circuit_discovery}, the distribution of neuron attribution values is heavily skewed, with the neuron identified by \citet{clementWeFoundNeuron2023}, marked in black, sitting within the top $99.8^\text{th}$ percentile of values.
This demonstrates that, rather than performing extensive ablation studies on all the neurons of GPT2-Large, querying the network's local mapping via \eqref{eq:neuron_attribution} effectively isolates functional circuits and filters out irrelevant neurons. 
In Appendix \ref{sec:neuron_attribution_robustness}, we use this example to demonstrate that \eqref{eq:neuron_attribution} is robust as an attribution metric.


\begin{figure}[ht]
    \centering
    \vspace{0pt}
    \begin{minipage}[b]{0.48\linewidth}
        \centering
        \includegraphics[width=\textwidth]{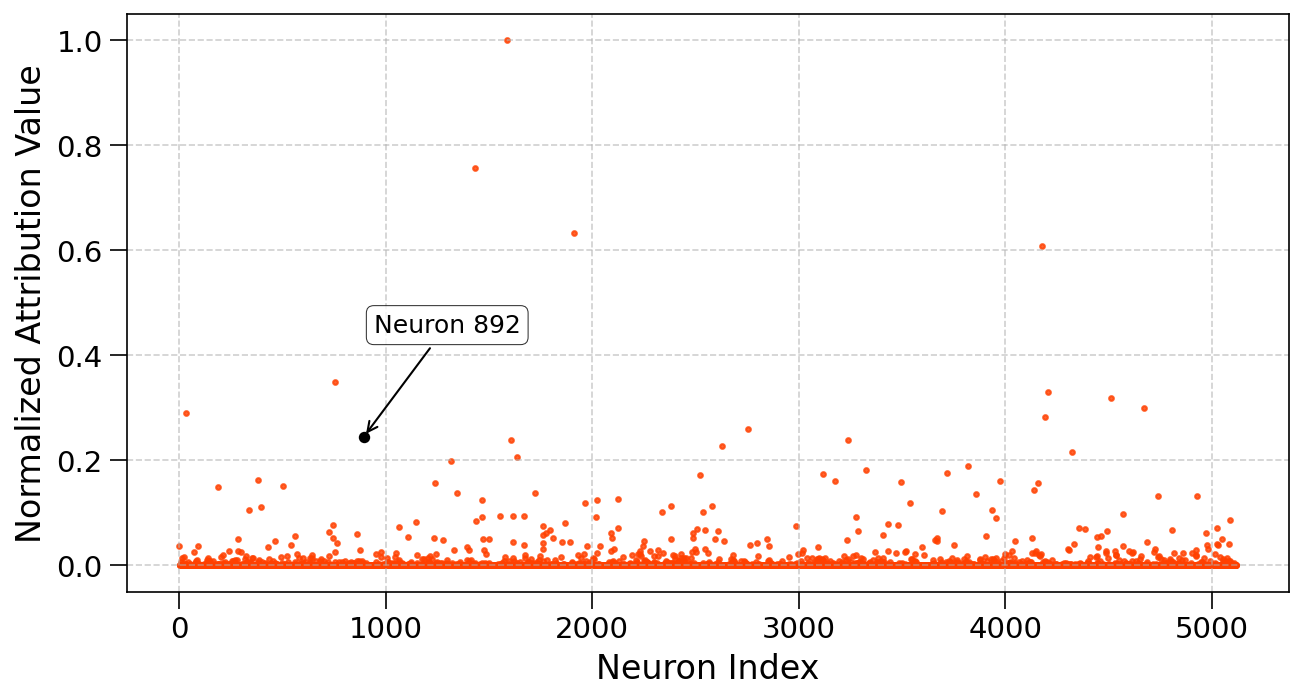}
        \caption{
        \textbf{Neuron-attribution metrics derived using LCH can quickly filter out irrelevant neurons, facilitating circuit discovery.}
        We prompt GPT2-Large and note the normalized attribution value (using \Cref{eq:neuron_attribution}) in a neighborhood of the embedding at the input of the multi-layer perceptron block at the thirty-first layer of the last token of this prompt. 
        The neighborhood is constructed by sampling $256$ points within a radius of $0.25$ of the embedding. 
        We normalize these values to the range $[0, 1]$. In black we indicate the $892^{\text{nd}}$ neuron in the multi-layer perceptron.}
        \label{fig:circuit_discovery}
    \end{minipage}\hfill
    \begin{minipage}[b]{0.48\linewidth}
        \centering
        \includegraphics[width=\textwidth]{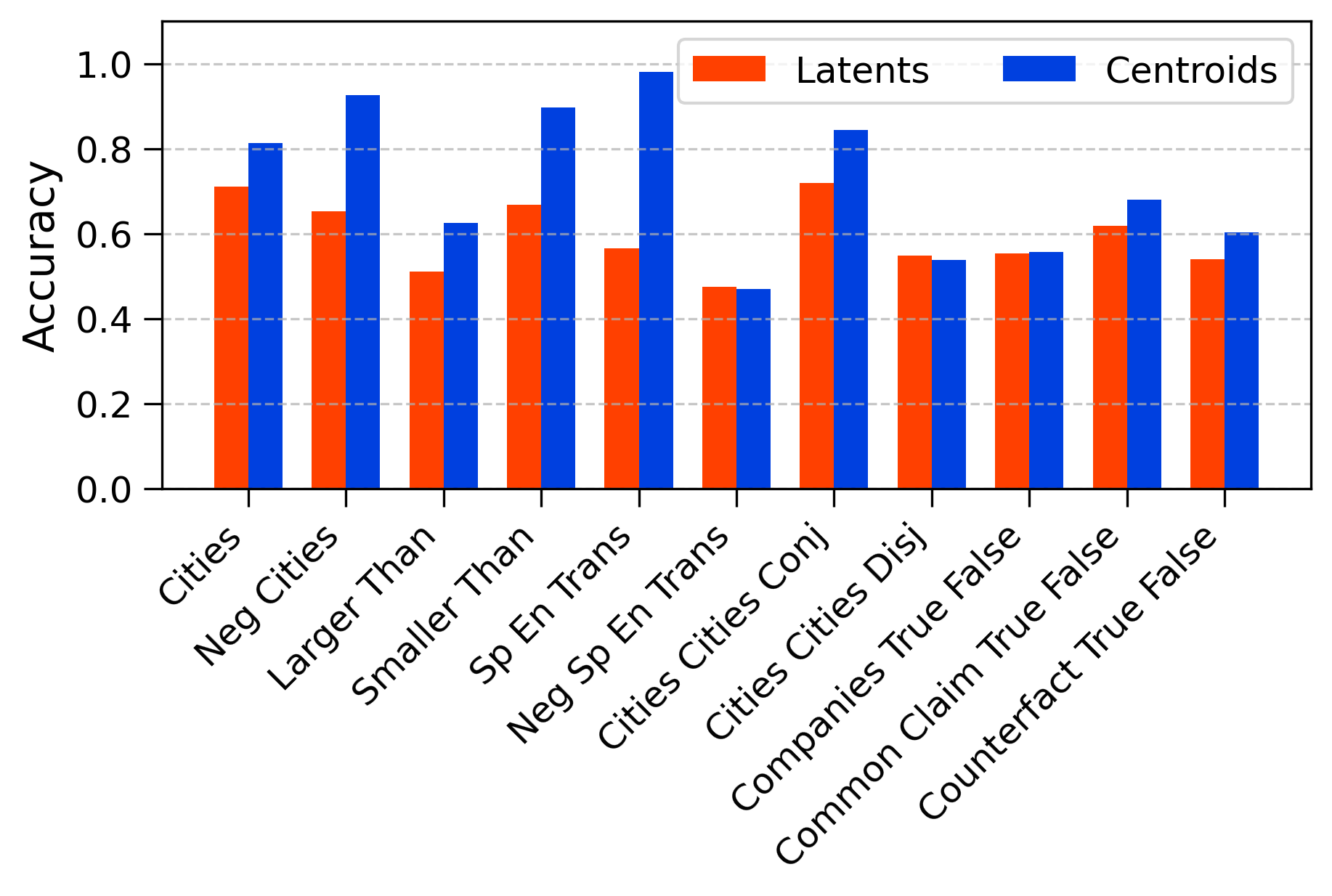}
        \caption{
        \textbf{Linear probes obtained using LCH exhibit greater generalization than those obtained using LRH.}
        We obtain mass-mean probes on the $\mathtt{likely}$ dataset from the twelfth layer of Llama-3.1-8B, either using intermediate activations or centroids extracted from the multi-layer perceptron component. 
        We then test these probes using other datasets from \citet{marksGeometryTruthEmergent2024}.}
        \label{fig:probes}
    \end{minipage}
    
\end{figure}

\subsection{Concept Discovery using Linear Probes}

Probing is another technique that exploits linear structures in DNs to either extract features \citep{kimInterpretabilityFeatureAttribution2018}, extract representations \citep{nandaEmergentLinearRepresentations2023}, or classify inputs \citep{marksGeometryTruthEmergent2024}. 
Here, we explore the latter of these applications by forming linear classifiers to discern the truthfulness of input statements to large language models. 
We adopt the mass-mean probes of \citet{marksGeometryTruthEmergent2024}, along with their datasets, to test the generalization capacity of these probes. 
Using the $\mathtt{likely}$ dataset, we obtain mass-mean probes from the twelfth layer of the Llama-3.1-8B large language model \citep{grattafioriLlama3Herd2024}, either using the latent activations or the centroids extracted from the multi-layer perceptron component. 
The $\mathtt{likely}$ dataset is constructed as a classification problem of whether sample tokens are likely or unlikely under the model's logit distribution for non-factual textual inputs.
The other datasets are formulated as classification problems between factually truthful and untruthful statements. Therefore, \textit{likely} is identifying plausibility in the model's outputs rather than a concept of truthfulness. 
Because centroid-based mass-mean probes generalize more effectively to these truth-identifying datasets (see Figure \ref{fig:probes}), it confirms that LCH captures the \textit{action} (the mapping) of outputting a truthful statement, rather than just the isolated concept. 
This underscores the intuition that LCH offers a mechanistic perspective by evaluating when the local experts of a DN align.


\section{Discussion}\label{sec:discussion}

The Linear Representation Hypothesis posits that feature discovery involves finding linear subspaces in the activation spaces of DN, but this critically decouples those features from the network's actual input-output map.
To make feature discovery mapping-aware, we introduced the Linear Centroids Hypothesis (LCH), which posits that feature discovery amounts to finding linear directions in {\em centroid space}.
Unlike intermediate representations, DN centroids are mapping-aware, hierarchically defined, intuitive artifacts of DNs that describe the DN's local experts.

Because centroids are efficiently accessible via Jacobian vector products, they serve as a drop-in replacement for activations in standard interpretability pipelines.
Consequently, we can demonstrate that LCH is operational and inherently resists the identification of spurious features.

Ultimately, LCH unifies feature dictionaries, probing, circuits, and saliency maps into a single geometric framework, ensuring that interpretability is mechanistically grounded in how the network actually computes.

\textbf{Limitations and Future Work.} The power diagram subdivision of a DN is parameterized by both centroids and radii. 
This study focuses entirely on the centroids, leaving the radii—which determine whether a centroid is contained within its own region—unexplored. 
Future research should investigate these radius parameters, as integrating them could provide an even more comprehensive geometric interpretation of the network's local experts and their decision boundaries.

Additionally, while this work introduced the \textit{local centroid} as a faithful, gradient-based saliency map to highlight relevant features for single inputs, future efforts should extend this application. 
Testing local centroids across a wider variety of DN architectures and broader input distributions will further validate their explanatory power within this mapping-aware framework.




\section*{Acknowledgments}

This work was supported by ONR grant N00014-23-1-2714, ONR MURI N00014-20-1-2787, DOE grant DE-SC0020345, and DOI grant 140D0423C0076.

\bibliographystyle{icml2026}
\bibliography{references}

\newpage
\appendix
\crefalias{section}{appendix}
\onecolumn

\section{Deep Network Geometry}\label{sec:dn_geometry}

The regions forming the geometry of a DN are constructed by the intersection of a collection of hyperplanes~\cite{montufarNumberLinearRegions2014,balestrieroMadMaxAffine2018,haninComplexityLinearRegions2019}.
Each nonlinearity (neuron) in a layer of the DN determines an activation level set within its input space (i.e., the hyperplane between being active or inactive in the case of the ReLU nonlinearity).
In a hierarchical fashion, starting from the first layer, these hyperplanes are pulled back to the input space of the DN and intersect to form the regions~\citep{balestrieroGeometryDeepNetworks2019,humayunSplineCamExactVisualization2023}.

The power diagram subdivision parametrization of the DN geometry, instead views the geometry as a hierarchical intersection of power diagrams~\cite{balestrieroGeometryDeepNetworks2019}.

\begin{definition}\label{def:power_diagrams}
    Given a collection of $Q$ centroid-radius pairs $\left\{\left(\vmu_q,\tau_q\right)\right\}_{q=1}^{Q}\subseteq\mathbb{R}^d\times\mathbb{R}$, a power diagram tessellates $\mathbb{R}^d$ into $Q$ disjoint tiles $\Omega=\{\omega_1,\dots,\omega_Q\}$ such that $\cup_{q=1}^Q \omega_q=\mathbb{R}^d$,
    with each tile given by
    \begin{equation}\label{eq:power_diagram}
        \omega_{q}=\left\{\vx\in\mathbb{R}^d:q = \argmin_{q' \in \{1,\dots,Q\}}\left(\left\Vert \vx - \vmu_{q'} \right\Vert_2^2 - \tau_{q'}\right)\right\}.
    \end{equation}
\end{definition}

The distance minimized in (\ref{eq:power_diagram}) is called the Laguerre distance \cite{imaiVoronoiDiagramLaguerre1985}, and differs from the Euclidean distance only in the addition of the weighting provided by the radius parameter.

The power diagram subdivision induced by the DN is then constructed recursively through layer-wise power diagrams.
The first layer of the DN partitions the input space $\mathbb{R}^{d^{(1)}}$ as $\Omega^{(1)}$ through a power diagram.
The second layer of the DN then partitions the projections of each tile $\omega^{(1)}\in\mathbb{R}^{d}$ in $\mathbb{R}^{d^{(1)}}$ induced by $f^{(1)}$ as a power diagram.
These are then pulled back to $\mathbb{R}^{d}$ to yield a finer partition of the input space $\Omega^{(1\leftarrow2)}$.
This partition is an example of a power diagram subdivision.
Continuing sequentially for each layer completes the power diagram subdivision $\Omega$ of the DN~\cite{balestrieroGeometryDeepNetworks2019}.

Just like for a regular power diagram, we can also associate each region in a power diagram subdivision with a {\em centroid} and {\em radius}.
Despite each tile in a power diagram subdivision being defined implicitly through a recursive combinatorial intersection of half-spaces, we will show that its centroid and radius are defined explicitly and computable independently of other regions.

We focus on centroids and how they can be analyzed to understand the {\em features} of a DN, though the radii are also of considerable interest.
Indeed, it is the presence of the radius that means that the centroid of a region in a power diagram need not be contained within the region itself.

\section{Local Centroids}\label{sec:local_centroids}

Here, we provide further illustrations of local centroids computed for state-of-the-art DNs.
In particular, we demonstrate that for randomly initialized DNs, local centroids contain no information (see \Cref{fig:saliency_map2}).
Furthermore, we show that considering DN sub-components from the input space to hidden layers facilitates the extraction of hierarchical saliency maps.
For example, in the top row of \Cref{fig:saliency_map3}, we compare local centroids computed from the input space to a hidden layer and from the input space to the output layer of a Swin-B transformer~\cite{liuSwinTransformerHierarchical2021}.
Since the windows of the aircraft are only apparent in the centroid of the full DN, it follows that the windows are a feature of the later layers of the DN.

\section{Comparison to Prior Work}\label{sec:comparison_to_literature}

\paragraph{Features.} 

Our notion of a DN feature (see \Cref{sec:features_and_circuits}) is analogous to the notion of a DN concept used in \citet{parkLinearRepresentationHypothesis2024}, which provides a rigorous theoretical characterization of the LRH. 
In \citet{parkLinearRepresentationHypothesis2024}, a DN feature is a variable that leads to a particular output when caused by a context. 
In our case, the context would correspond to the input samples that possess a given characteristic.
The variable notion of \citet{parkLinearRepresentationHypothesis2024} would then be equivalent to our requirement \emph{representation}.
Then the referenced causation on the output would be equivalent to our requirement of causing an \emph{influence}.

\paragraph{Circuits.} 
Circuits were initially introduced in \citet{olahZoomIntroductionCircuits2020} to deal with the apparent poly-semantic nature of neurons. 
That is, specific neurons were observed to trigger on seemingly semantically disjoint inputs, whereas ensembles of neurons demonstrated more reliable activation patterns. 
Instead, our notion of a circuit arises naturally as the responses of the components of a DN to a feature. 
Indeed, this response is likely to involve multiple neurons or components of a DN, given the DN's compositional construction.

\section{The Implication of the LCH on Hyperplane Geometry}\label{sec:formal_theory}

In this section, we formalize the relationship between aligned centroids and the hyperplane geometry of DNs described in Appendix \ref{sec:dn_geometry}. 
We first prove that at any given layer, centroids lying on a one-dimensional affine subspace imply that the boundaries separating their regions are strictly parallel. 

\begin{proposition}\label{prop:linear_parallel}
    At the input space of layer $\ell$, if the centroids of a sequence of adjacent linear regions lie on a one-dimensional affine subspace, then the hyperplanes forming the boundaries between these consecutive regions are strictly parallel.
\end{proposition}

\begin{proof}
    Consider a sequence of adjacent linear regions $\omega_1,\dots,\omega_k$ in the input space of layer $\ell$, separated by boundary hyperplanes $\Pi_1,\dots,\Pi_{k-1}$. 
    Let their corresponding power diagram centroids be $\vmu_1,\dots,\vmu_k$. 
    The shared boundary $\Pi_i$ between region $i$ and $i+1$ has a normal vector defined by the difference of their centroids: $\vn_i=\vmu_i-\vmu_{i+1}$.
    Assume the centroids lie on a one-dimensional affine subspace. Then, there exists a base point $\vp$ and a unit direction vector $\vd$ such that for all $i$, $\vmu_i = \vp + t_i \vd$ for some scalar $t_i$. 
    The normal vector for the boundary $\Pi_i$ is then:$$\vn_i = \vmu_i - \vmu_{i+1} = (t_i - t_{i+1})\vd.$$
    Because every normal vector $\vn_i$ is a scalar multiple of the exact same direction $\vd$, all boundary hyperplanes $\Pi_1, \dots, \Pi_{k-1}$ share the same normal direction and are therefore strictly parallel.
\end{proof}

While Proposition \ref{prop:linear_parallel} establishes parallel boundaries in the intermediate representation space of layer $\ell$, the network's overall behavior is determined by how these boundaries map back to the original input space $\mathbb{R}^d$.

The parallel hyperplanes $\Pi_1, \dots, \Pi_{k-1}$ at layer $\ell$ form a tightly packed, linear decision band. 
The boundary in the original input space corresponding to $\Pi_i$ is its pre-image under this mapping $\left\{\vx \in \mathbb{R}^d : f^{(1 \leftarrow \ell-1)}(\vx) \in \Pi_i\right\}$.
Because $f^{(1 \leftarrow \ell-1)}$ is a composition of affine transformations and non-linearities (like ReLU), it fundamentally acts by folding, stretching, and compressing the input space.

\begin{figure}[ht]
     \centering
     \begin{subfigure}[b]{0.4\columnwidth}
         \centering
         \includegraphics[width=0.8\textwidth]{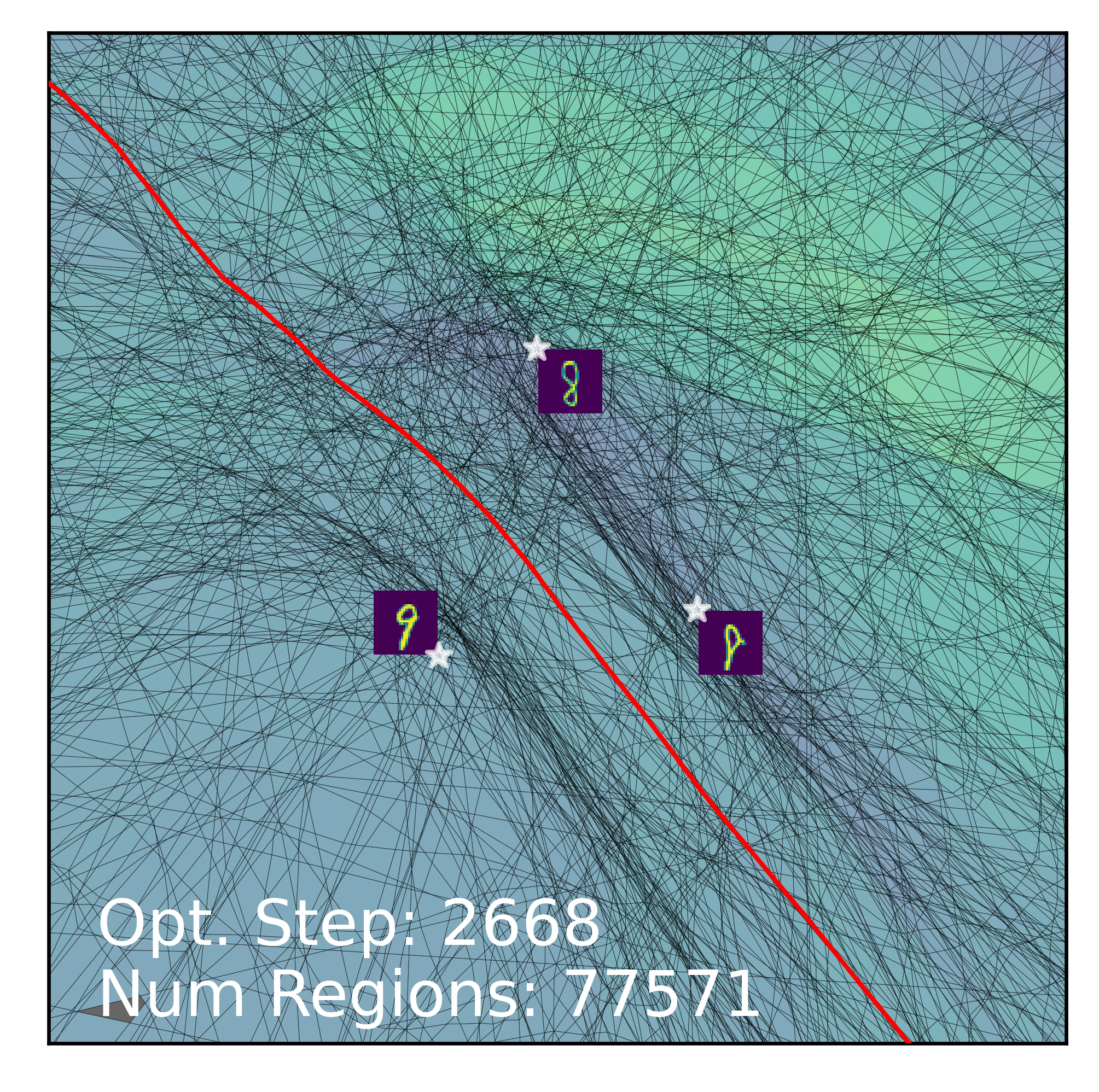}
         \caption*{Low Robustness}
         \label{fig:general_partition}
     \end{subfigure}
     \hspace{0.5em}
     \begin{subfigure}[b]{0.4\columnwidth}
         \centering
         \includegraphics[width=0.8\textwidth]{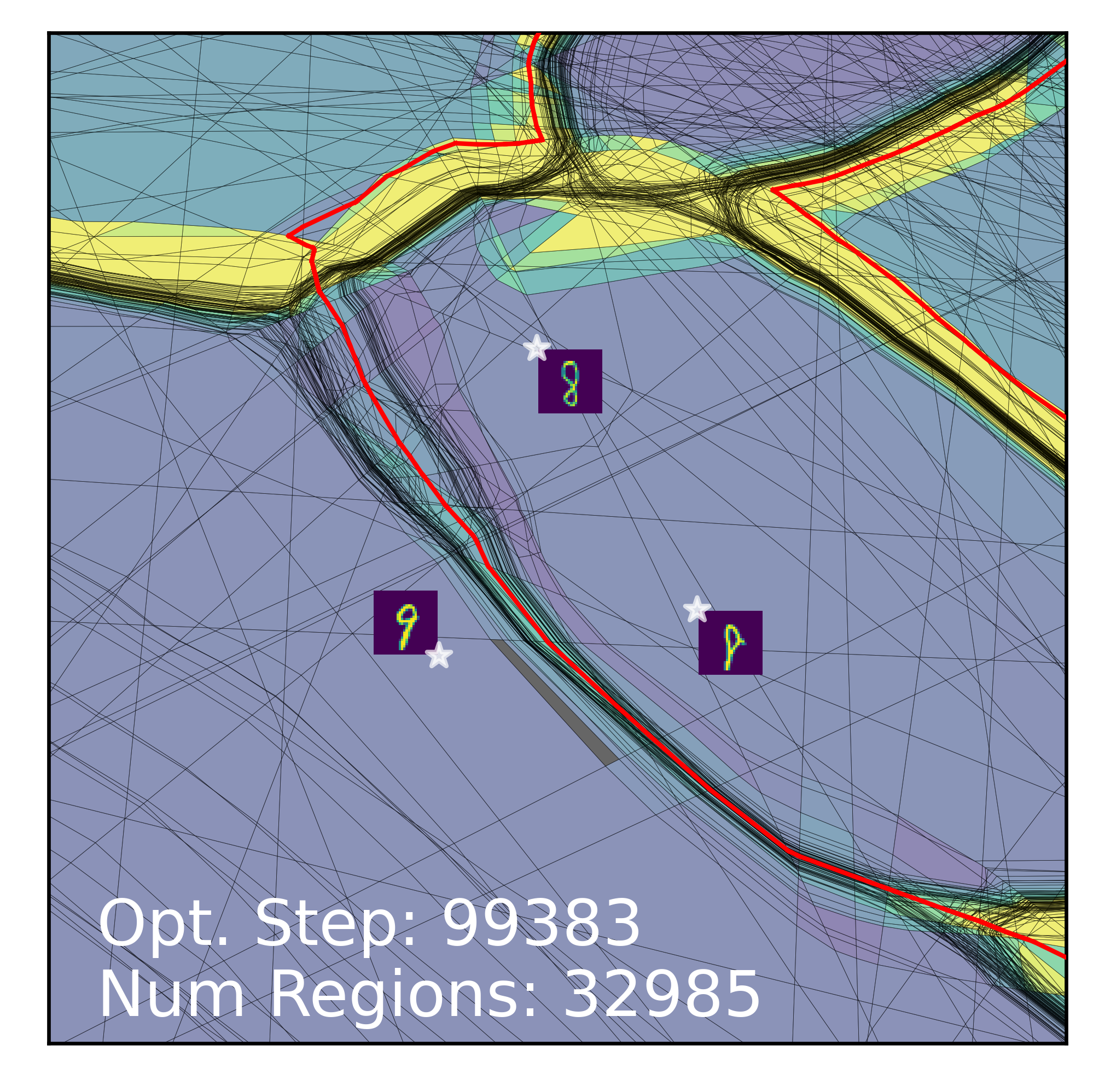}
         \caption*{High Robustness}
         \label{fig:robust_partition}
     \end{subfigure}
    \caption{
    Taken from~\citet{humayunSplineCamExactVisualization2023}, this is an illustration of the geometry of a fully connected ReLU DN as it trains on MNIST.
    Three training samples are fixed to construct a two-dimensional slice of the DN's input domain, and the linear regions that intersect this slice are then visualized.
    On the \textbf{left} is the geometry of the DN when it has generalized but exhibits low levels of robustness.
    On the \textbf{right} is the geometry of the DN when it exhibits the additional property of robustness.
    }
    \label{fig:splinecam}
\end{figure}

When we ``pull back'' the parallel hyperplanes $\Pi_i$ through this mapping, a single straight hyperplane in layer $\ell$ pulls back to a continuous, piecewise linear surface in the input space. 
Each time the pre-image crosses an activation boundary from an earlier layer, the surface bends, allowing it to approximate macroscopic curves.
Because the hyperplanes $\Pi_1, \dots, \Pi_{k-1}$ are parallel and sequentially packed in layer $\ell$, their pre-images map to a dense bundle of piecewise linear surfaces in the input space.
In regions where the mapping $f^{(1 \leftarrow \ell-1)}$ highly compresses the input manifold, this bundle of boundaries packs densely together. 
By accumulating these intricately folded, tightly packed linear regions, the network effectively constructs smooth, arbitrarily curved decision boundaries to separate complex features in the input data.

Consequently, the linear alignment of centroids at an intermediate layer is the geometric mechanism by which a DN coordinates a massive accumulation of linear regions to trace out curved, semantically meaningful feature boundaries in the original input space.
An example of this is shown in Figure \ref{fig:splinecam}.

\section{Analyzing Deep Networks Trained to Classify Interiors of Polygons}\label{sec:other_polygons}

Here we continue the analysis of the DN training to classify the interior of a star-shaped polygon, as depicted in the third and fourth panels of \Cref{fig:1}.

First, we can observe in Figure \ref{fig:centroid_progression} that the linearity of the centroids of \Cref{fig:1} emerges gradually through training. 
At initialization, the centroids have a similar arrangement to the input samples because the DN is randomly initialized. 
However, as training progresses, we observe that the centroids slowly migrate and align themselves. 
In particular, we can see the alignment of the centroids manifest before they reach their final positions.

\begin{figure}[h]
     \centering
     \begin{subfigure}[b]{0.19\textwidth}
         \centering
         \includegraphics[width=\textwidth]{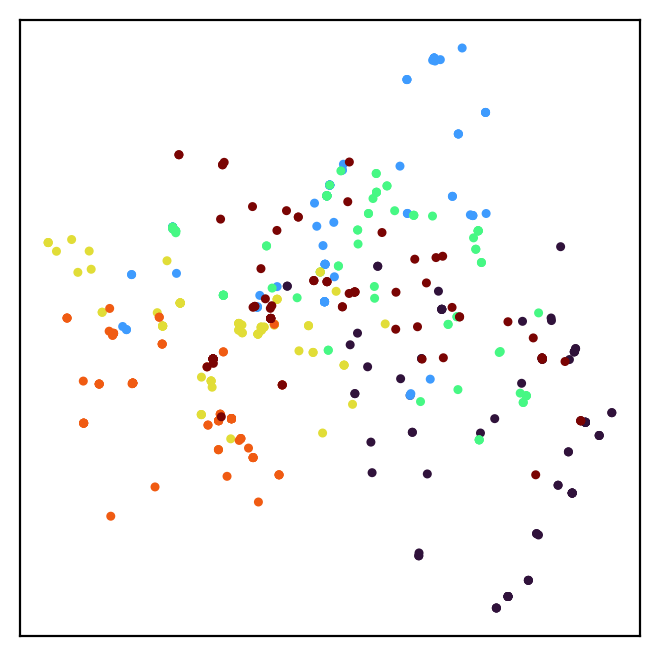}
         \caption*{Epoch $0$}
     \end{subfigure}
     \hfill
     \begin{subfigure}[b]{0.19\textwidth}
         \centering
         \includegraphics[width=\textwidth]{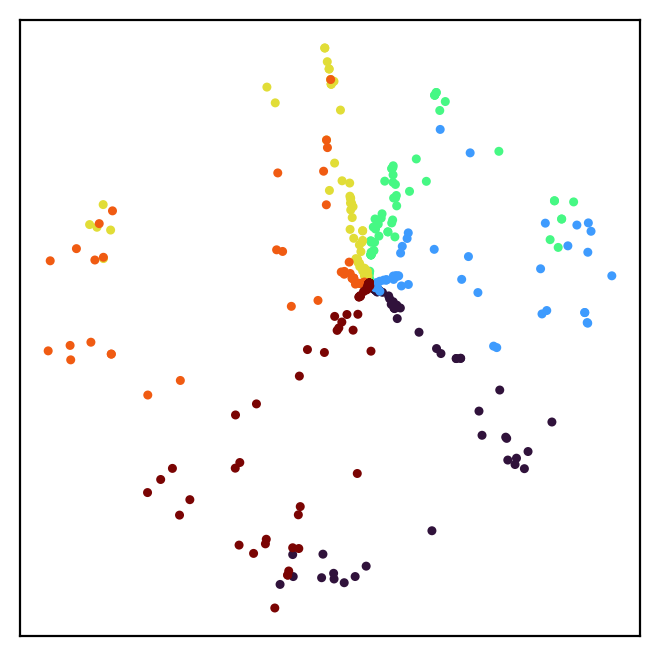}
         \caption*{Epoch $18$}
     \end{subfigure}
     \hfill
     \begin{subfigure}[b]{0.19\textwidth}
         \centering
         \includegraphics[width=\textwidth]{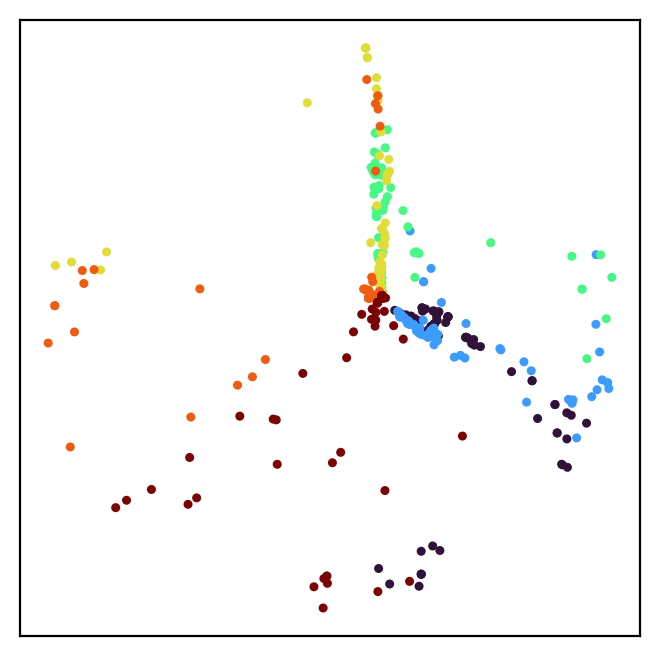}
         \caption*{Epoch $27$}
     \end{subfigure}
     \hfill
     \begin{subfigure}[b]{0.19\textwidth}
         \centering
         \includegraphics[width=\textwidth]{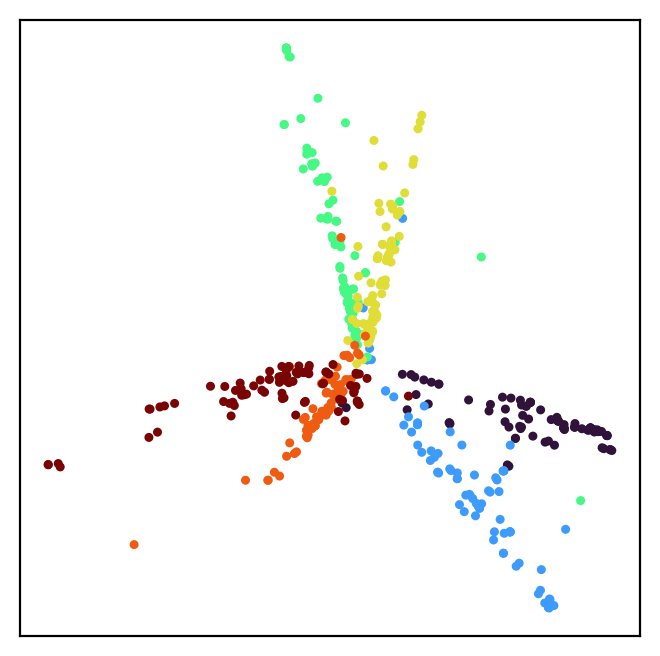}
         \caption*{Epoch $147$}
     \end{subfigure}
     \hfill
     \begin{subfigure}[b]{0.19\textwidth}
         \centering
         \includegraphics[width=\textwidth]{figures/star_edge_relu_centroids.png}
         \caption*{Epoch $512$}
     \end{subfigure}
    \caption{
    Throughout the training of the DN of the bottom row \Cref{fig:1}, we tracked the DN centroids of the input samples highlighted in the bottom left panel \Cref{fig:1}.}
    \label{fig:centroid_progression}
\end{figure}

Second, as expected from the discussion of \Cref{sec:formal_theory}, in the second panel of Figure \ref{fig:toy_example} we see that the trained DN geometry exhibits linear regions with roughly parallel boundaries along the polygon's edges.
Moreover, in the third and fourth panels of Figure \ref{fig:toy_example}, we observe the activation level-sets of the nonlinearities of the second and third hidden layers of this DN.
It is clear that the full identification of the interior and exterior of the polygon is achieved in the third layer, as the nonlinearities delineate the polygon's edges. 
The second-layer nonlinearities capture the coarser features relating to the three exterior sectors created by the star shape.
This is evidenced by each nonlinearity forming along one of the polygon's outer edges.

\begin{figure}[ht]
     \centering
     \begin{subfigure}[b]{0.24\columnwidth}
         \centering
         \includegraphics[width=\textwidth]{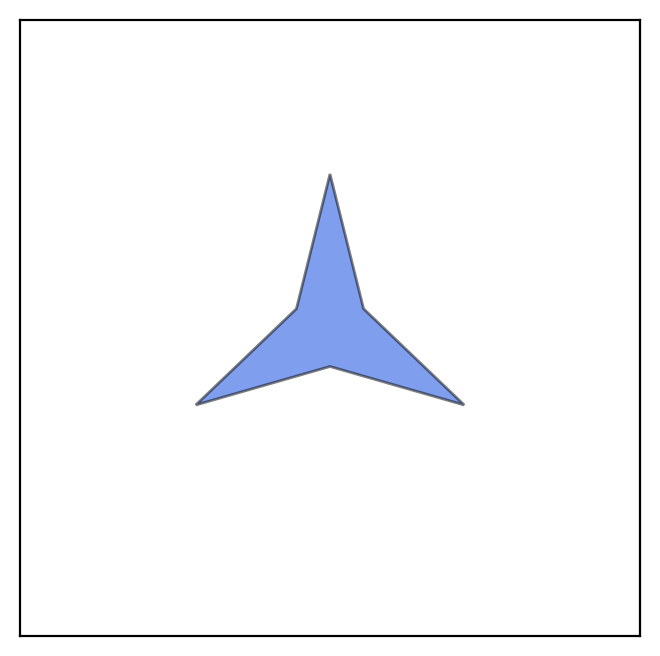}
         \caption*{Input Polygon}
         \label{fig:star_polygon}
     \end{subfigure}
     \begin{subfigure}[b]{0.24\columnwidth}
         \centering
         \includegraphics[width=\textwidth]{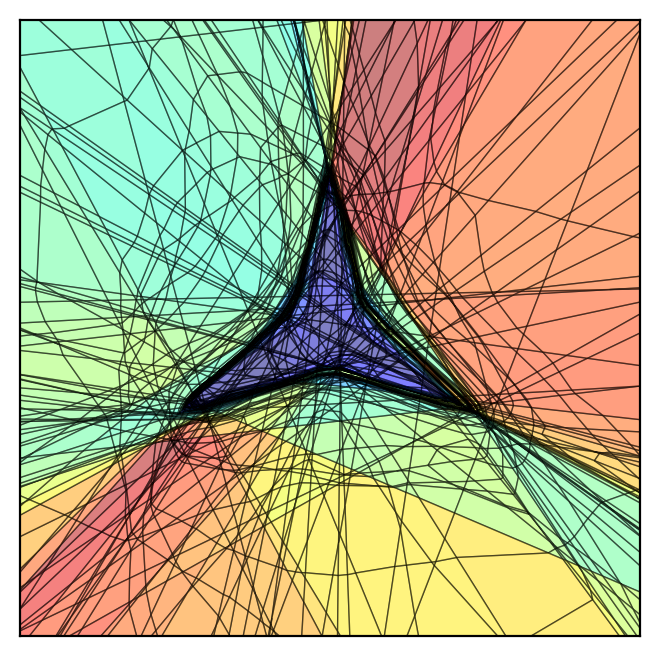}
         \caption*{DN Geometry}
         \label{fig:partition}
     \end{subfigure}
     \begin{subfigure}[b]{0.24\columnwidth}
         \centering
         \includegraphics[width=\textwidth]{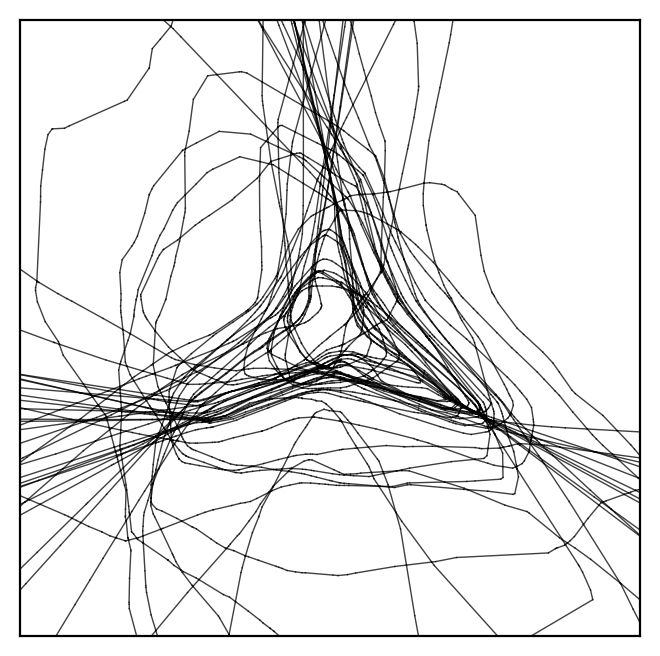}
         \caption*{$2^{\text{nd}}$ Hidden Layer}
         \label{fig:second_partition_layer}
     \end{subfigure}
     \begin{subfigure}[b]{0.24\columnwidth}
         \centering
         \includegraphics[width=\textwidth]{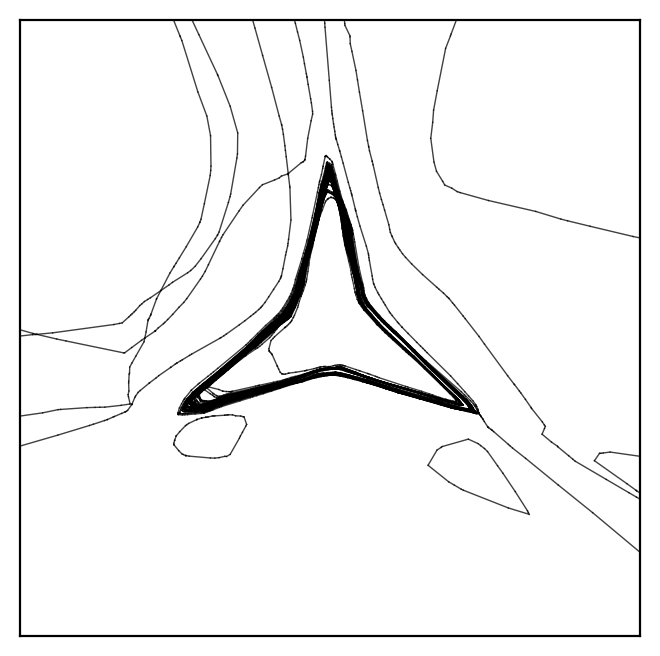}
         \caption*{$3^{\text{rd}}$ Hidden Layer}
         \label{fig:third_partition_layer}
     \end{subfigure}
    \caption{
    The activation level-sets of the nonlinearities of a DN have a structure that is related to features of a DN.
    Here we train a fully connected ReLU DN with a two-dimensional input space and a one-dimensional output space to classify the interior and exterior of the star-shaped polygon shown in the first panel.
    In the second panel, we visualize the geometry of the trained DN.
    The color of the linear regions represents the Frobenius norm of the affine transformation parameter $\mA_{\omega}$ operating on that linear region.
    We then separate the hyperplanes corresponding to the nonlinearities of the second layer (third panel) and those of the third layer (fourth panel).
    }
    \label{fig:toy_example}
\end{figure}

Third, in addition to centroids possessing a linear structure under the LCH, it is evident from \Cref{fig:1} that this structure is semantically coherent.
Suggesting that applications of these topological point-cloud analyses may be particularly fruitful under the LCH.
To aid the study of centroids as point clouds, we soften DNs using CPA nonlinearities (e.g., ReLU) by replacing them with smooth approximations (e.g., GELU). 
The GELU nonlinearity~\citep{hendrycksGaussianErrorLinear2023} belongs to the swish family of nonlinearities \citep{ramachandranSearchingActivationFunctions2017}, which are theoretically known to provide an appropriate softening of a ReLU DN's geometry \citep{balestrieroHardSoftUnderstanding2018}.
In the first panel of Figure \ref{fig:point_clouds}, we show that this smooth approximation does not impact the structure of the centroids.
In the second panel of \Cref{fig:point_clouds}, we consider a t-SNE embedding \citep{maatenVisualizingDataUsing2008} of these softened centroids at the second hidden layer of the DN.
We observe clustering of centroids by the sector of the input domain in which the input samples were located.
This corroborates our prior analysis using the second hidden layer's geometry.
Moreover, it demonstrates how the LCH can be applied hierarchically.
In \Cref{fig:1}, centroids were computed across the entire DN to understand the finest features of the DN, whereas in Figure \ref{fig:toy_example}, centroids were computed across a single layer to identify coarser features.
Similar analyses under the LRH are challenging, as there is no notion of latent activations across different scales of sub-components~\cite{balaganskyMechanisticPermutabilityMatch2025}.

Fourth, we can analyze the neurons of the DN using Equation (\ref{eq:neuron_attribution}).
In the third panel of Figure \ref{fig:point_clouds}, we see that the neurons of the third hidden layer have an encompassing effect on the entire boundary of the polygon,

\begin{figure}[ht]
     \centering
     \begin{subfigure}[b]{0.32\columnwidth}
         \centering
         \includegraphics[width=\textwidth]{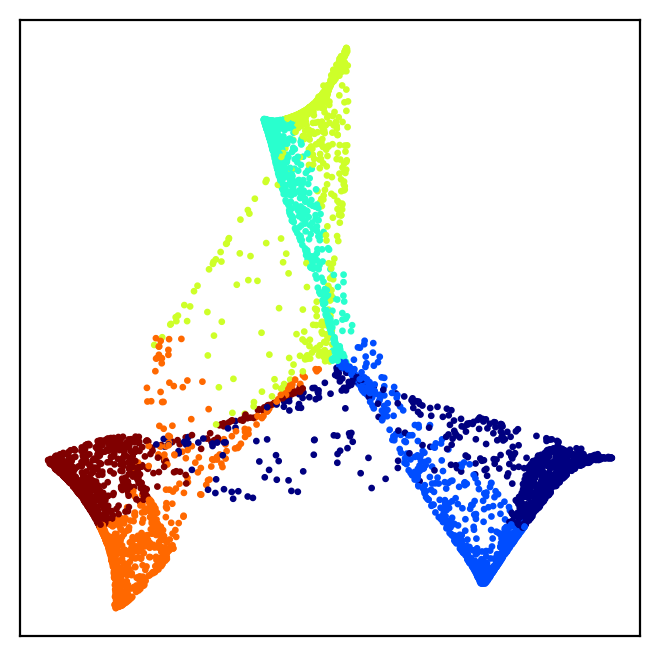}
         \caption*{GELU Centroids}
         \label{fig:star_grid_gelu}
     \end{subfigure}
     \begin{subfigure}[b]{0.32\columnwidth}
         \centering
         \includegraphics[width=\textwidth]{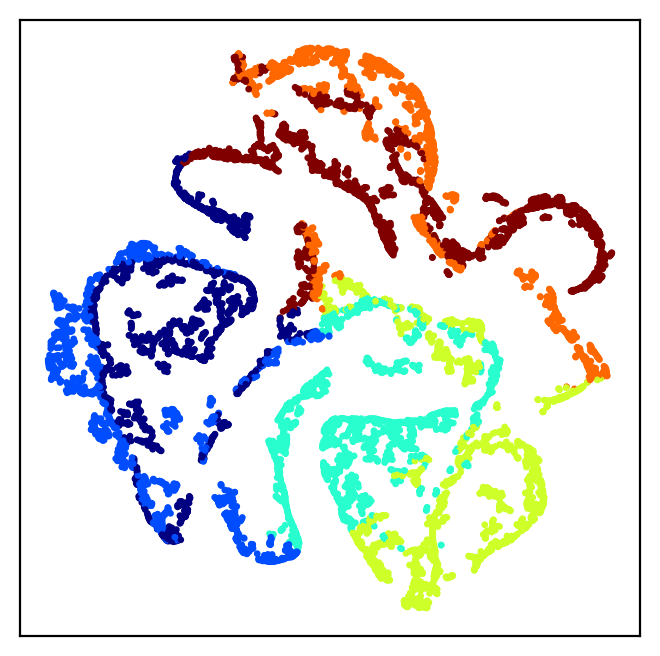}
         \caption*{$2^\text{nd}$ Layer Centroids t-SNE}
         \label{fig:tsne}
     \end{subfigure}
     \hspace{0.5em}
     \begin{subfigure}[b]{0.32\columnwidth}
         \centering
         \includegraphics[width=\textwidth]{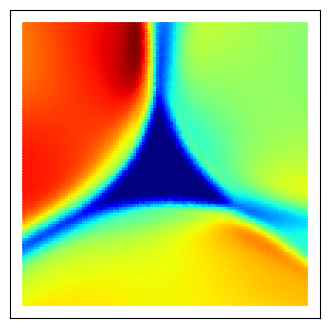}
         \caption*{$s^{(i,3)}$}
         \label{fig:neuron_layer3}
     \end{subfigure}
    \caption{
    In the first panel, we show that replacing the ReLU nonlinearities with GELU nonlinearities for the DN of Figure \ref{fig:toy_example} does not change the structure of the centroids significantly.
    In the second panel, we show, using t-SNE analysis of the second-layer centroids of the DN in Figure \ref{fig:toy_example}, that the second layer encodes features corresponding to the three external sectors of the star-shaped polygon.
    In the third panel, we show that the third layer of the DN of Figure \ref{fig:toy_example} captures the inter-exterior feature by demonstrating that neurons in this layer are sensitive (as per Equation (\ref{eq:neuron_attribution})).
    More specifically, for a grid of points in the input space, we apply \Cref{eq:neuron_attribution} to a small neighborhood around each point and visualize the resulting value as a heatmap.
    }
    \label{fig:point_clouds}
\end{figure}

In addition to the star-shaped polygon considered in the main text, in \Cref{fig:other_shapes} we corroborate the observed patterns when the input distribution is a bowtie-shaped and reuleaux-shaped polygon.



Moreover, we can show that starting from a non CPA DN, we can repeat the same investigations and derive similar conclusions.
For example, as showing in Figure \ref{fig:gelu_from_scratch}, when replacing the nonlinearity back to a ReLU we can observe its geometry using SplineCam \citep{humayunSplineCamExactVisualization2023} and see that nonlinearities still align along the boundary of the polygon, when we observe the centroids of input samples we see the same linear structures, and the influence of pruning neurons on the centroids is still effective as a neuron attribution metric.

\begin{figure}[h]
     \centering
     \begin{subfigure}[b]{0.24\textwidth}
         \centering
         \includegraphics[width=\textwidth]{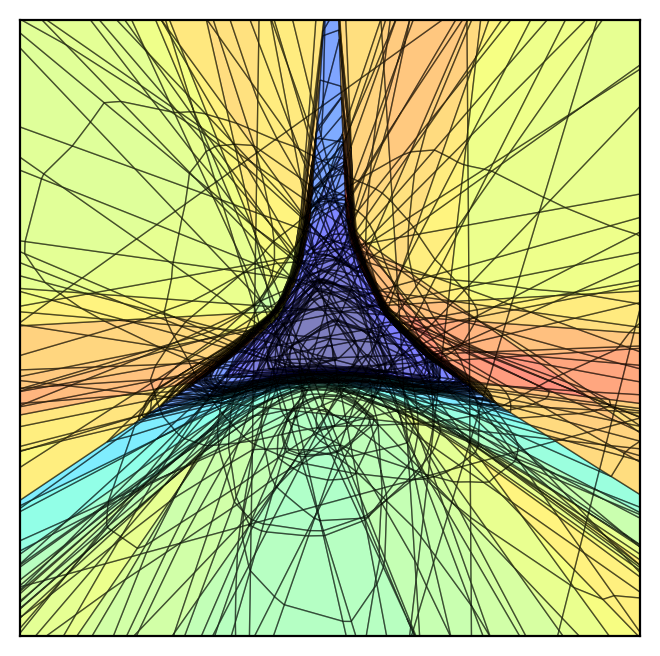}
         \caption*{DN Geometry}
     \end{subfigure}
     \begin{subfigure}[b]{0.24\textwidth}
         \centering
         \includegraphics[width=\textwidth]{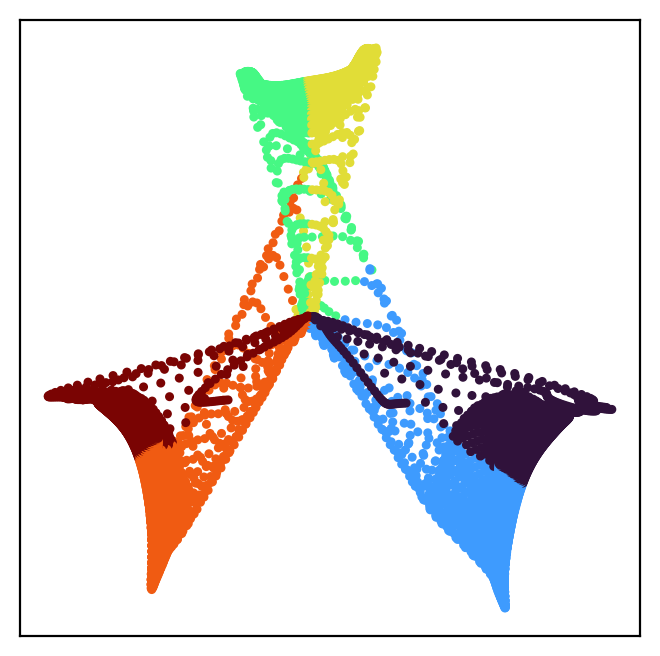}
         \caption*{Centroids}
     \end{subfigure}
     \begin{subfigure}[b]{0.235\textwidth}
         \centering
         \includegraphics[width=\textwidth]{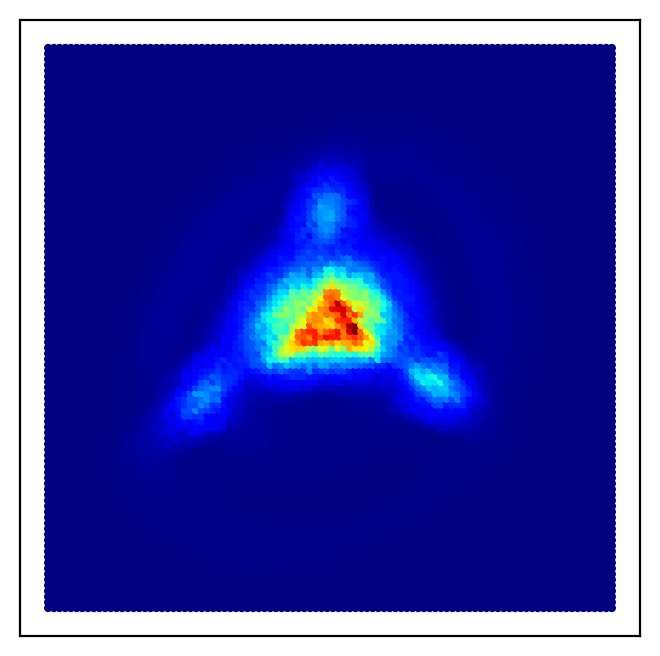}
         \caption*{$s^{(i,2)}$}
     \end{subfigure}
     \begin{subfigure}[b]{0.235\textwidth}
         \centering
         \includegraphics[width=\textwidth]{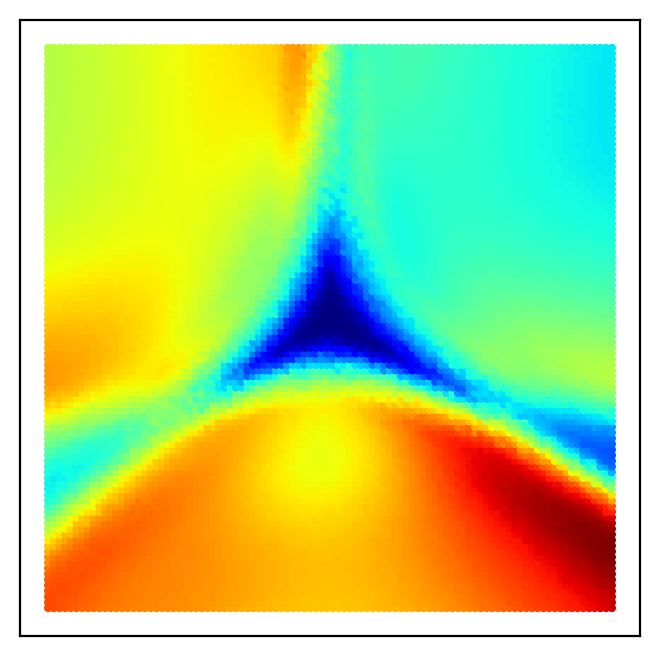}
         \caption*{$s^{(i,3)}$}
     \end{subfigure}
    \caption{
    Here we train a DN in the same manner as the one considered in Figure \ref{fig:toy_example}, except we use the GELU nonlinearity. 
    In the first panel, we replace the nonlinearities with ReLU such that we can use SplineCam to visualize its geometry. 
    In the second panel, we visualize the centroids from input samples. 
    In the third and fourth panels, we consider the sensitivities of centroids when pruning neurons from the second and third layers, respectively.}
    \label{fig:gelu_from_scratch}
\end{figure}

\section{Qualitative Analysis of Feature Dictionaries}\label{sec:qualitative}

Here, we provide a qualitative analysis of the feature dictionaries constructed in \Cref{sec:saes}.
To address the concern that Imagenette may provide a relatively easy setting, we train sparse autoencoders on a subset of ImageNet sampled from ten random dog breeds.

In the first panel of Figure \ref{fig:qualitative_saes}, we can see that even in this more challenging setting, the features learned under LCH generalize better, just like in Figure \ref{fig:centroid_saes}.

Moreover, with the right panel of Figure \ref{fig:qualitative_saes}, we can see qualitatively that the features identified using LCH are more semantically coherent.
Here, we identify similar samples from the test set by computing the Jaccard similarity of their feature activation patterns.
We create a grid of images, with the first image as a fixed sample and the remaining six as its closest neighbors by Jaccard similarity.
In both cases, the features identify the same dog breed; however, only under LCH do we consistently see a dog with the same head position.

\begin{figure}[ht]
    \centering
    \begin{minipage}[c]{0.4\textwidth}
        \centering
        \begin{tikzpicture}
        \begin{groupplot}[
            group style={
                group size=1 by 1,
                horizontal sep=1.4cm,
            },
            width=0.9\linewidth,
            height=4.5cm,
            grid=major,
            grid style={dashed, gray!30},
            legend style={font=\scriptsize, cells={anchor=west}},
            tick label style={font=\scriptsize},
            label style={font=\scriptsize},
            title style={font=\bfseries}
        ]
        \nextgroupplot[
            xlabel={K},
            ylabel={Probe Accuracy},
            xmode=log,
            log basis x=2,
            xtick={8,16,32,64},
            xticklabels={8,16,32,64},
            ymin=55, ymax=75,
            legend pos=north west,
        ]
        \addplot[color=myblue, mark=*, thick] 
            table[x=k, y=Centroids, col sep=comma] {data/probe_accuracy_vs_sparsity-dogs.csv};
        \addlegendentry{LCH}

        \addplot[color=myorange, mark=square*, thick, dashed] 
            table[x=k, y=Latents, col sep=comma] {data/probe_accuracy_vs_sparsity-dogs.csv};
        \addlegendentry{LRH}

        \end{groupplot}
        \end{tikzpicture}
    \end{minipage}%
    \hfill
    \begin{minipage}[c]{0.55\textwidth}
        \centering
        \begin{subfigure}[b]{\linewidth}
            \caption*{LRH}
            \includegraphics[width=\linewidth]{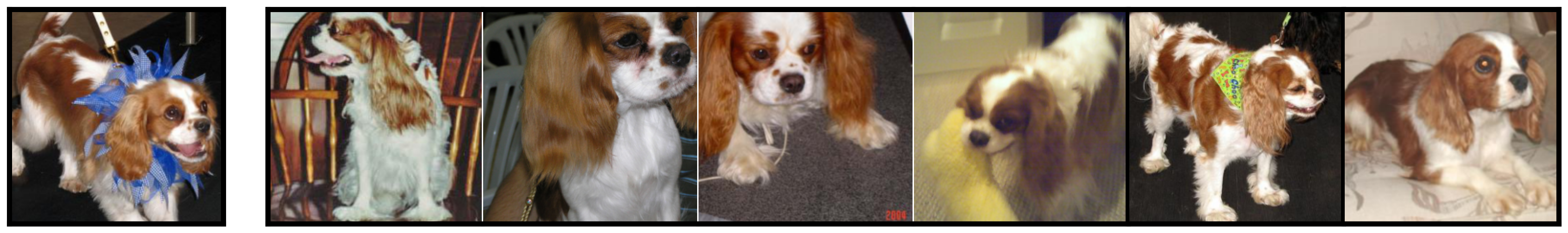}
        \end{subfigure}
        \begin{subfigure}[b]{\linewidth}
            \caption*{LCH}
            \includegraphics[width=\linewidth]{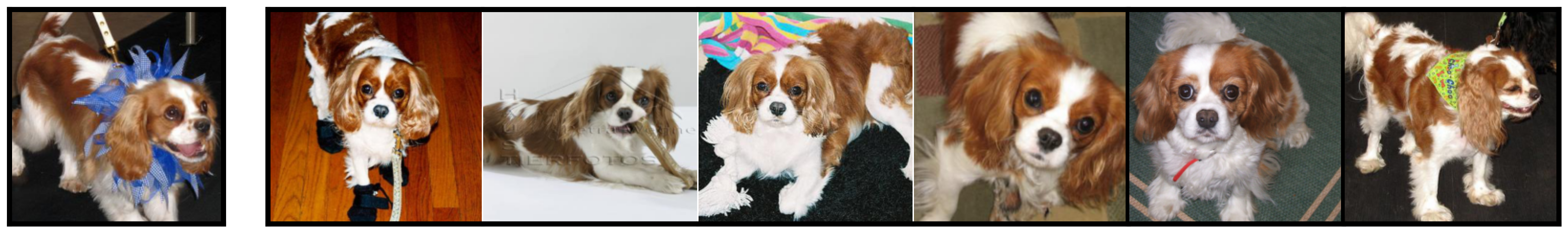}
        \end{subfigure}
    \end{minipage}
    
    \caption{
        Here we train sparse autoencoders in the same ways as Figure \ref{fig:centroid_saes}, but on a subset of ImageNet~\cite{krizhevskyImageNetClassificationDeep2012} containing ten random dog breed classes, which we split into a train and test set.
        In the first panel, we evaluate the sparse autoencoders by training linear probes on the feature activation of the training set and measuring their accuracy on the test set.
        In the right panel, we sample from the test set and compute the Jaccard similarity between its feature activation pattern and those of the other samples in the test set.
        We visualize the sample in the left image in the grid, then identify the six most similar samples using the other images in the grid.
    }
    \label{fig:qualitative_saes}
\end{figure}

\section{Validating the Platonic Representation Hypothesis}\label{sec:prh}

The Platonic Representation Hypothesis (PRH) states that as DNs become larger and more capable, the features they form converge~\citep{huh2024position}.
With Figure \ref{fig:prh}, we demonstrate that analyzing centroids in a similar manner also supports the PRH. In Figure \ref{fig:prh}, we adopt the methodology of~\citet{huh2024position} and compare the cosine similarities of activations (i.e., latent activations or centroids) computed across layers of BLOOM language models~\cite{workshop2023bloom176bparameteropenaccessmultilingual} of increasing capacity to those computed across the DINOv2 vision transformer~\citep{oquabDINOv2LearningRobust2024}.
As expected under the PRH, the cosine similarity of activations and model capacity is positively related.

\begin{figure}[ht]
    \centering
    \includegraphics[width=0.5\linewidth]{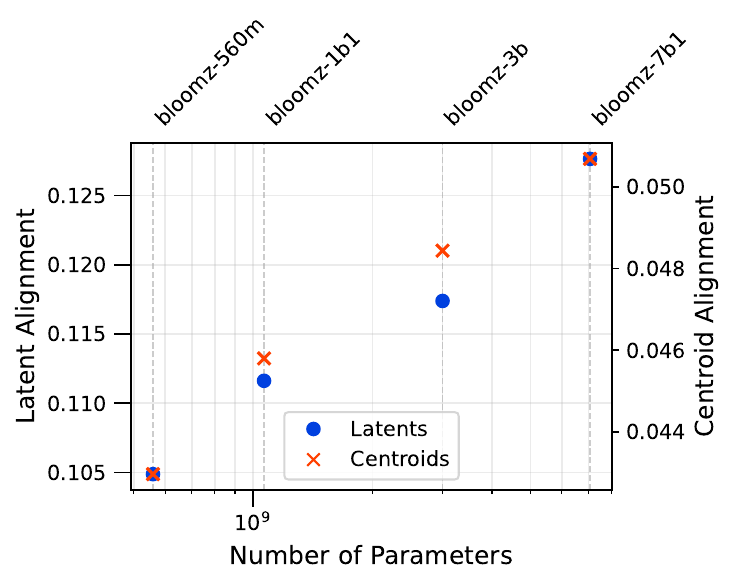}
    \caption{
    The Platonic Representation Hypothesis is validated by observing the cosine similarities between the centroids of DNs of different capacities.
    We compare the alignment of centroids from the BLOOM language model to the centroids of the DINOv2 vision model.}
    \label{fig:prh}
\end{figure}

A limitation of this analysis is that it considers only the maximum cosine similarities among all possible activations of one DN and all possible activations of another. 
It does not account for how the full spectrum of features in one model aligns with those of another.
This is rectified by our analysis shown in the third panel of Figure \ref{fig:centroid_saes}, which demonstrates that the LCH more faithfully captures this hypothesis.

\section{Computational Requirements.}\label{sec:compute_requirements}

A valid concern with exploring the LCH is the computational burden it imposes on interpretability, as it requires evaluating the DN's Jacobians. 
Fortunately, this interrogation only requires considering Jacobian vector products (see Proposition \ref{prop:centroid_jacobian}), which are significantly cheaper to compute in common computational frameworks. 
Furthermore, the analysis of centroids often focuses on a relatively small component of the DN.
Thus, this computation would appear relatively insignificant compared to processing the entire DN. 
In this section, we empirically quantify the computational burden of using centroids rather than latents in our main experiments in \Cref{sec:experiments}.

\paragraph{FashionMNIST Color Correlation.}

This experiment took approximately 2 hours on a Quadro RTX 8000.
For each correlation value (six in total), five separate DNs were trained.
Two linear probes were then trained to classify the color feature, one using intermediate activations and the other using centroids.

The DN comprised a three-layer convolutional feature extractor followed by a two-layer multi-layer perceptron classifier.
Intermediate activations were extracted from the input to the classifier, and centroids were computed using the input-output Jacobian of the classifier.

The DN was trained for 10 epochs using the Adam~\cite{kingmaAdamMethodStochastic2017} optimizer at a learning rate of $0.001$.
The linear probe was trained for 3 epochs, again with the Adam optimizer at a learning rate of $0.001$.

\paragraph{DINO Feature Extraction.}

For this experiment, we find almost no difference in the time necessary to extract centroids.
It only takes $8.7$ seconds compared to extracting intermediate representations, which takes $7.8$ seconds.
After these vectors are extracted, the computational pipeline is identical when using centroids or intermediate activations. 

In the first panel of Figure \ref{fig:centroid_saes}, we train TopK sparse autoencoders~\cite{gaoScalingEvaluatingSparse2025} on all token positions of DINOv2 and sparsity values $8$, $16$, $32$, and $64$.
Training each sparse autoencoder for 40 epochs takes around 1 hour on an NVIDIA TITANX.
In the second panel of Figure \ref{fig:centroid_saes}, we consider the sparse autoencoders with sparsity value $32$.

In the third of Figure \ref{fig:centroid_saes}, we consider similar sparse autoencoders but trained on only the \textit{cls} token of DINOv2/v3.
Training each of these sparse autoencoders takes about 2 minutes on an NVIDIA TITANX.

\paragraph{GPT2 Circuit Discovery.}

Although there is no direct analog of this experiment with latent activations, we can still argue that the computational burden is relatively benign. 
In particular, since we only compute centroids across the multi-layer perceptron block of the thirty-first layer, we only need to consider the Jacobian vector product for this component. 
This can be done by storing the gradients from a forward pass across this block, which, relative to performing a forward pass across the model, is insignificant.
This experiment takes less than 1 hour on an NVIDIA TITANX.

\paragraph{Llama-3.1-8B Probes.}

To implement this experiment, we use the setup of \url{https://github.com/saprmarks/geometry-of-truth}.
The only difference involves the extraction of centroids in addition to the extraction of intermediate activations.
Extracting centroids took $2329$ seconds compared to $470$ seconds for intermediate activations.




\section{Robustness of Neuron Attribution with Centroids}\label{sec:neuron_attribution_robustness}

In order for \Cref{eq:neuron_attribution} to prove useful as a tool for interpreting DNs, it is essential that it is robust in its application.
For example, \Cref{eq:neuron_attribution} ought not be sensitive to the neighborhood $\mathcal{N}$ chosen.
Furthermore, since in practice we can only approximate \Cref{eq:neuron_attribution} by taking a finite sample of points from $\mathcal{N}$, it is important that \Cref{eq:neuron_attribution} has low variance in relation to this finite sample.
In \Cref{fig:neuron_attribution_robustness}, we test both these properties for the experiment of \Cref{fig:circuit_discovery}. In the left plot, we observe that attribution values have a low variance when a finite sample is used to approximate the neighborhood $\mathcal{N}$.
In the right plot, we observe that the percentile of a particular neuron within the DN layer is stable across different neighborhood sizes.
This ensures that conclusions derived from \Cref{eq:neuron_attribution} are robust.

\section{Model Licenses}\label{sec:licenses}

Swin-B~\cite{liuSwinTransformerHierarchical2021}, ResNet~\cite{heDeepResidualLearning2016} and ConvNeXt-L~\cite{liuConvNet2020s2022} are obtained through PyTorch~\cite{torchvision2016} and RobustBench~\cite{croceRobustBenchStandardizedAdversarial2020} under an MIT license.

DINOv2~\cite{oquabDINOv2LearningRobust2024} is used under an Apache 2.0 license, DINOv3~\cite{simeoniDINOv32025} under the DINOv3-license, the Bloom language models~\cite{workshop2023bloom176bparameteropenaccessmultilingual} under the BigScience RAIL license, Llama-3.1-8B under the Llama3.1 license, GPT-2~\cite{radfordLanguageModelsAre2019} is obtained through the nano-GPT\footnote{\url{https://github.com/karpathy/nanogpt}} repository under an MIT license.

\begin{figure}[p]
     \centering

     \begin{subfigure}[b]{0.24\textwidth}
         \centering
         \caption*{\tiny\centering Input\\Sample}
         \includegraphics[width=\textwidth]{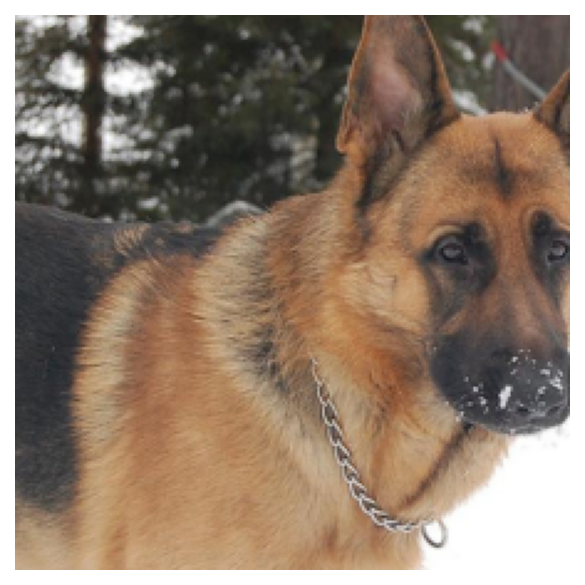}
     \end{subfigure}
     \hfill
     \begin{subfigure}[b]{0.24\textwidth}
         \centering
         \caption*{\tiny\centering Local Centroid\\(Input Space to Hidden Layer)}
         \includegraphics[width=\textwidth]{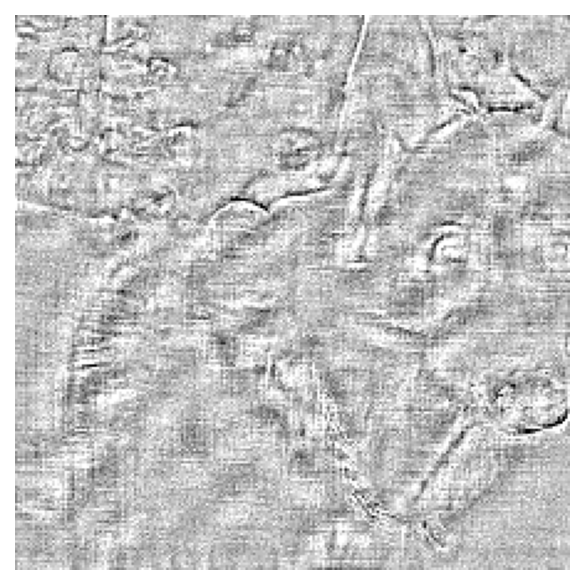}
     \end{subfigure}
     \begin{subfigure}[b]{0.24\textwidth}
         \centering
         \caption*{\tiny\centering Local Centroid\\(Input Space to Output)}
         \includegraphics[width=\textwidth]{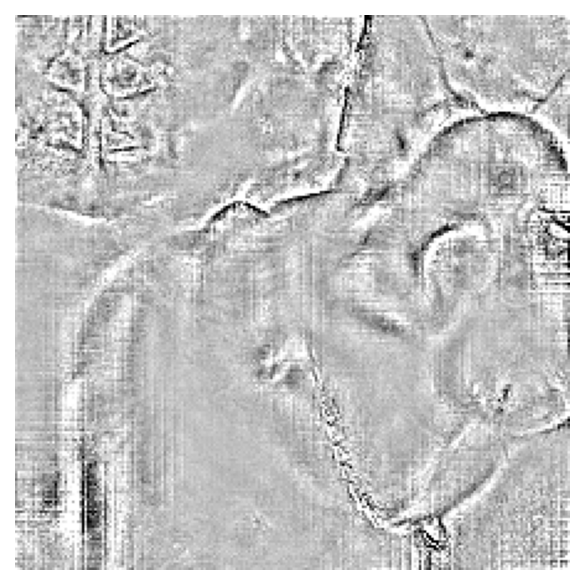}
     \end{subfigure}
     \hfill
     \begin{subfigure}[b]{0.24\textwidth}
         \centering
         \caption*{\tiny\centering Local Centroid\\(At Random Initialization)}
         \includegraphics[width=\textwidth]{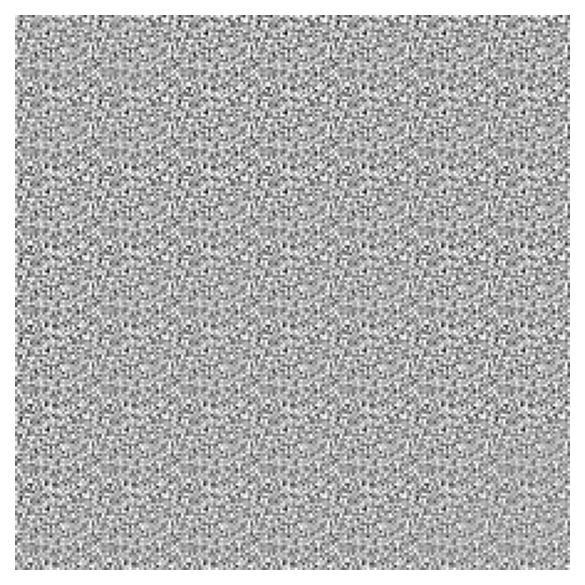}
     \end{subfigure}

     \begin{subfigure}[b]{0.24\textwidth}
         \centering
         \includegraphics[width=\textwidth]{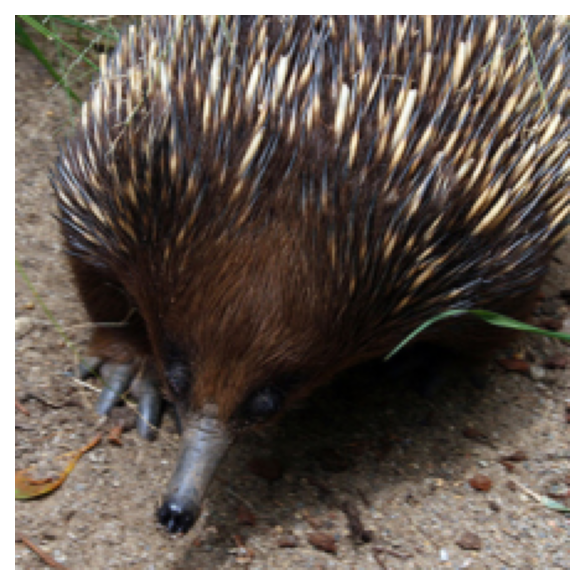}
     \end{subfigure}
     \hfill
     \begin{subfigure}[b]{0.24\textwidth}
         \centering
         \includegraphics[width=\textwidth]{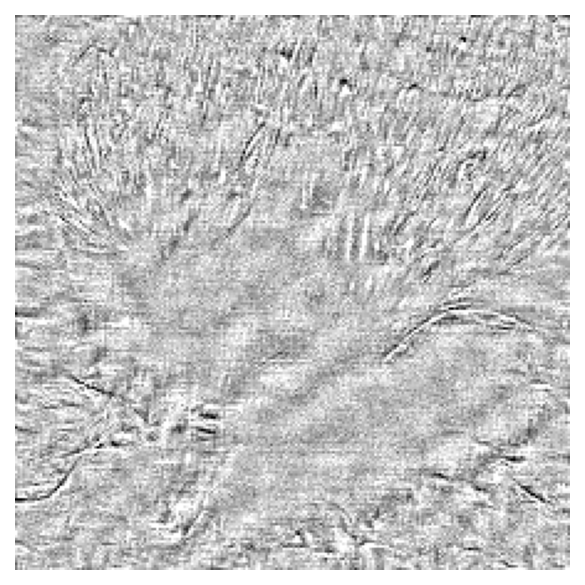}
     \end{subfigure}
     \begin{subfigure}[b]{0.24\textwidth}
         \centering
         \includegraphics[width=\textwidth]{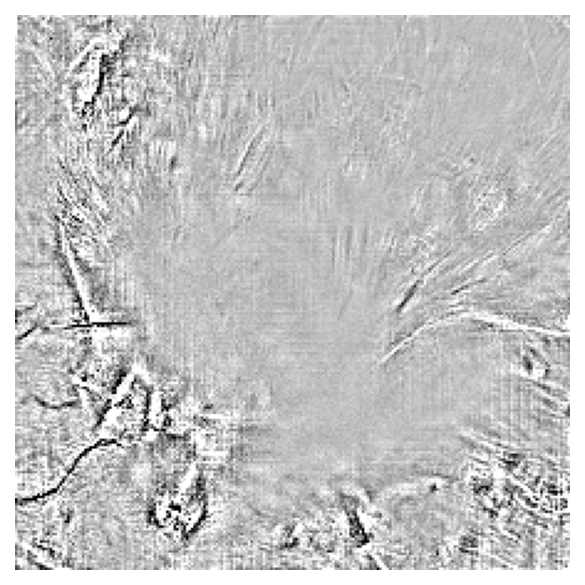}
     \end{subfigure}
     \hfill
     \begin{subfigure}[b]{0.24\textwidth}
         \centering
         \includegraphics[width=\textwidth]{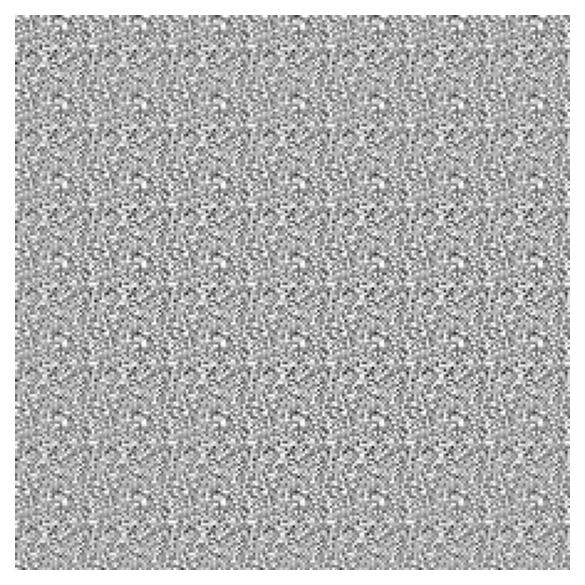}
     \end{subfigure}

     \begin{subfigure}[b]{0.24\textwidth}
         \centering
         \includegraphics[width=\textwidth]{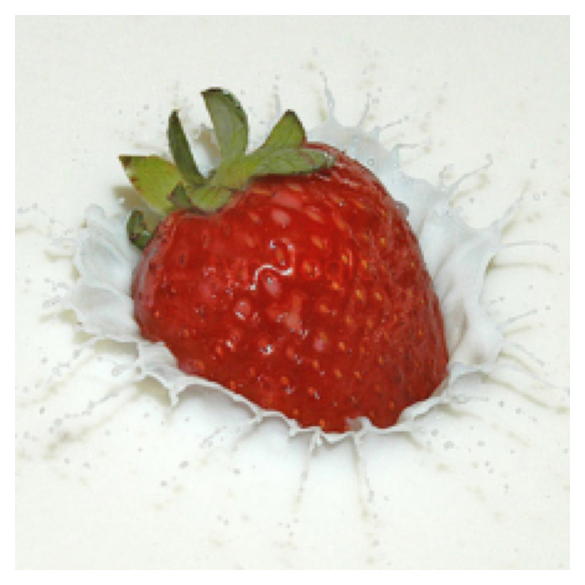}
     \end{subfigure}
     \hfill
     \begin{subfigure}[b]{0.24\textwidth}
         \centering
         \includegraphics[width=\textwidth]{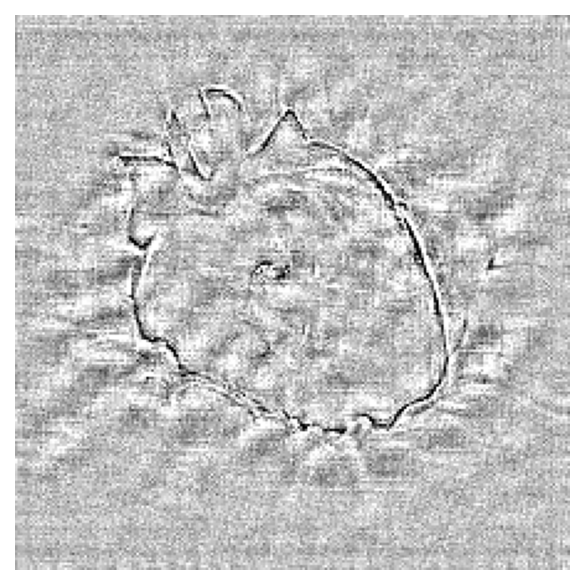}
     \end{subfigure}
     \begin{subfigure}[b]{0.24\textwidth}
         \centering
         \includegraphics[width=\textwidth]{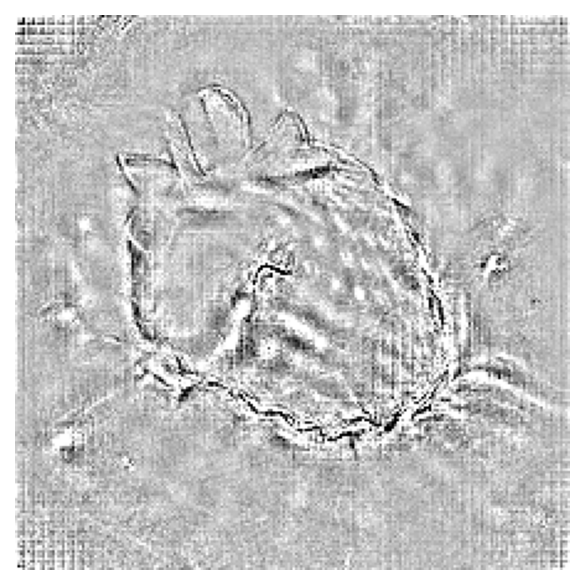}
     \end{subfigure}
     \hfill
     \begin{subfigure}[b]{0.24\textwidth}
         \centering
         \includegraphics[width=\textwidth]{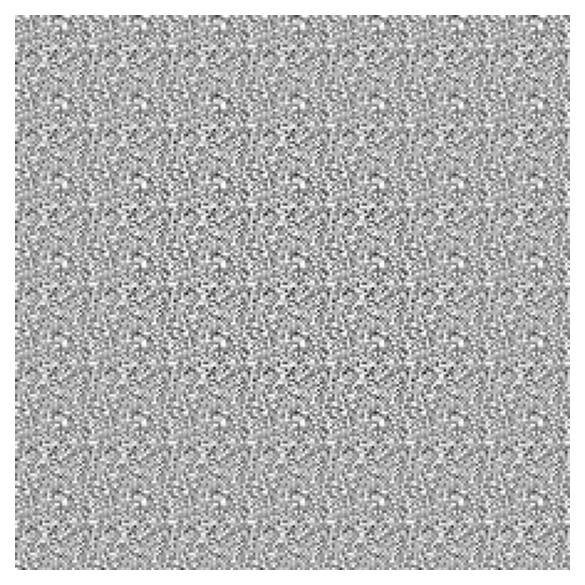}
     \end{subfigure}

     \begin{subfigure}[b]{0.24\textwidth}
         \centering
         \includegraphics[width=\textwidth]{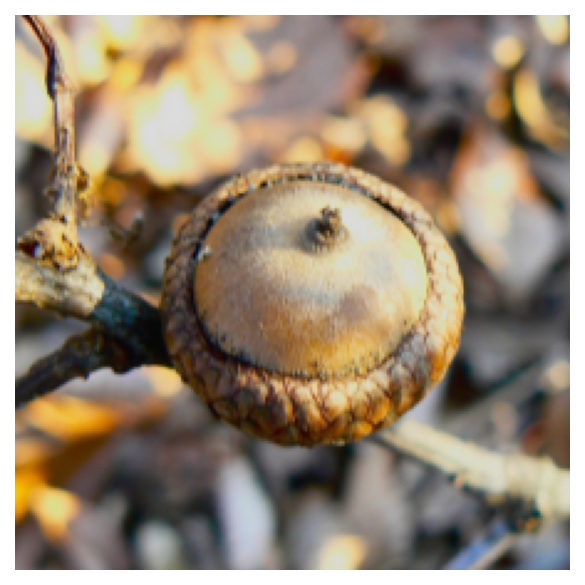}
     \end{subfigure}
     \hfill
     \begin{subfigure}[b]{0.24\textwidth}
         \centering
         \includegraphics[width=\textwidth]{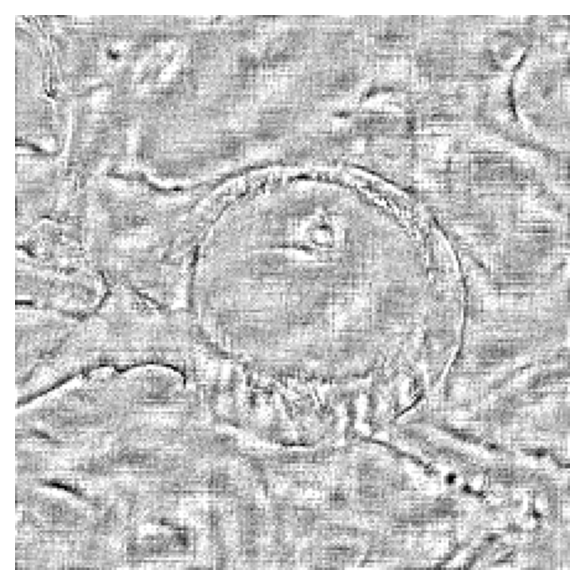}
     \end{subfigure}
     \begin{subfigure}[b]{0.24\textwidth}
         \centering
         \includegraphics[width=\textwidth]{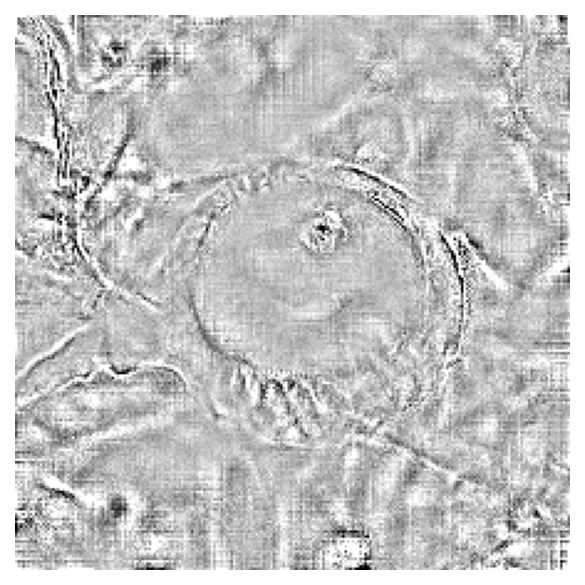}
     \end{subfigure}
     \hfill
     \begin{subfigure}[b]{0.24\textwidth}
         \centering
         \includegraphics[width=\textwidth]{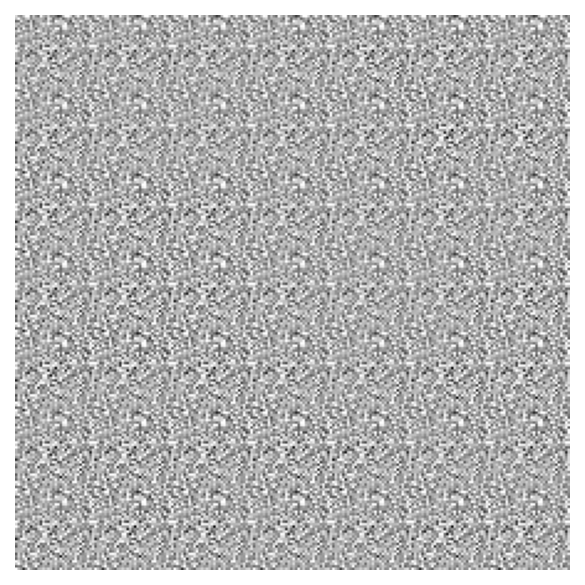}
     \end{subfigure}
     
    \caption{
    The local centroids of the pre-trained ConvNext-L DN from PyTorch.
    }
    \label{fig:saliency_map2}
\end{figure}

\begin{figure}[p]
     \centering
     \begin{subfigure}[b]{0.32\textwidth}
         \centering
         \caption*{\tiny Input}
         \includegraphics[width=\textwidth]{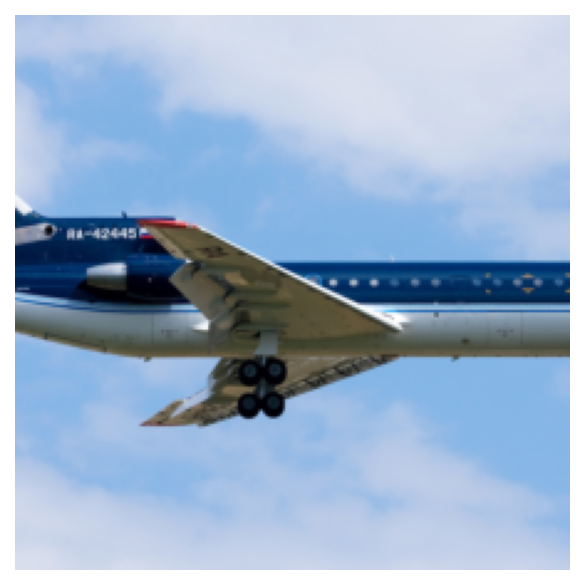}
     \end{subfigure}
     \hfill
     \begin{subfigure}[b]{0.32\textwidth}
         \centering
         \caption*{\tiny Local Centroid (Input Space to Hidden Layer)}
         \includegraphics[width=\textwidth]{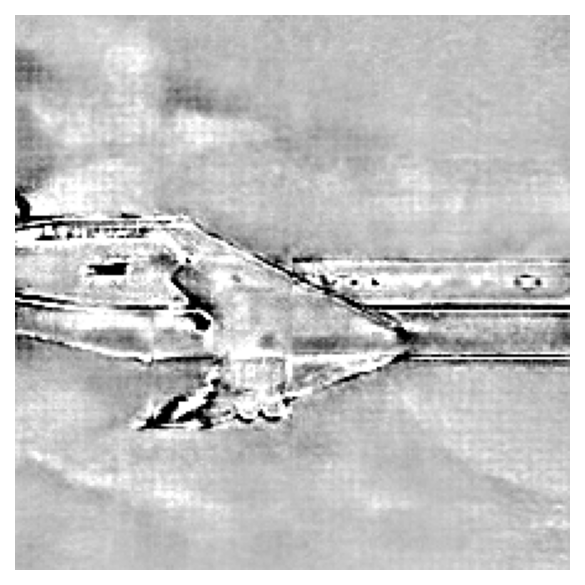}
     \end{subfigure}
     \begin{subfigure}[b]{0.32\textwidth}
         \centering
         \caption*{\tiny Local Centroid (Input Space to Output)}
         \includegraphics[width=\textwidth]{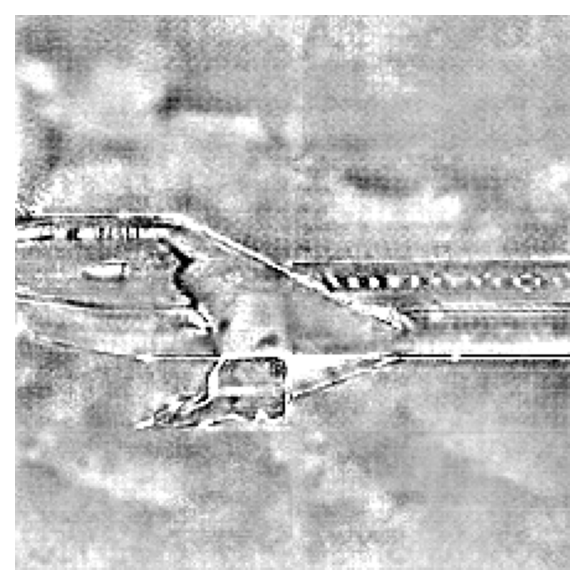}
     \end{subfigure}

     \begin{subfigure}[b]{0.32\textwidth}
         \centering
         \includegraphics[width=\textwidth]{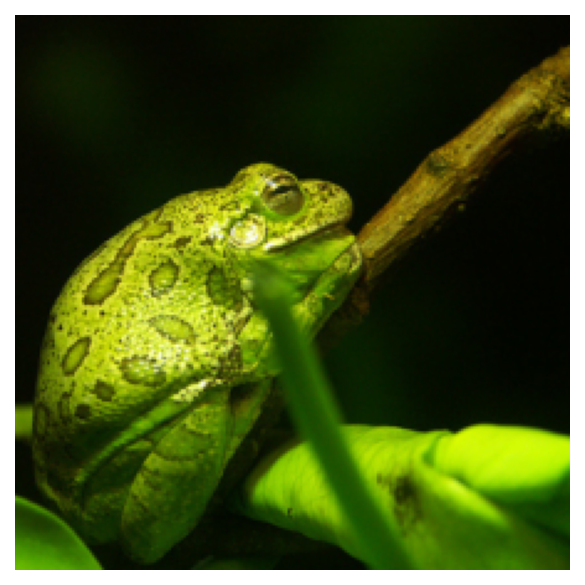}
     \end{subfigure}
     \hfill
     \begin{subfigure}[b]{0.32\textwidth}
         \centering
         \includegraphics[width=\textwidth]{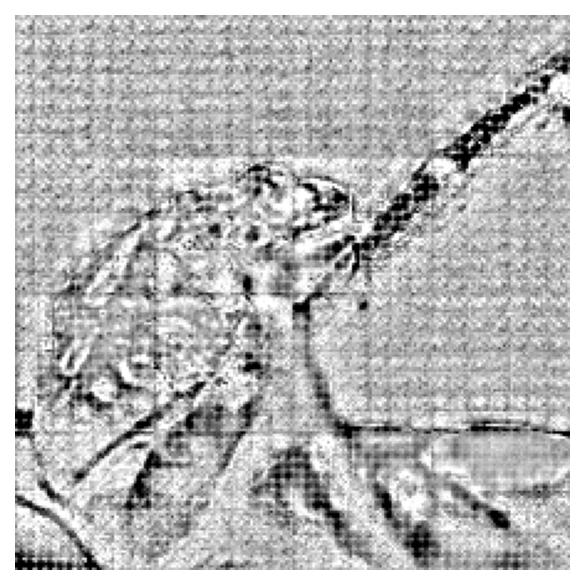}
     \end{subfigure}
     \begin{subfigure}[b]{0.32\textwidth}
         \centering
         \includegraphics[width=\textwidth]{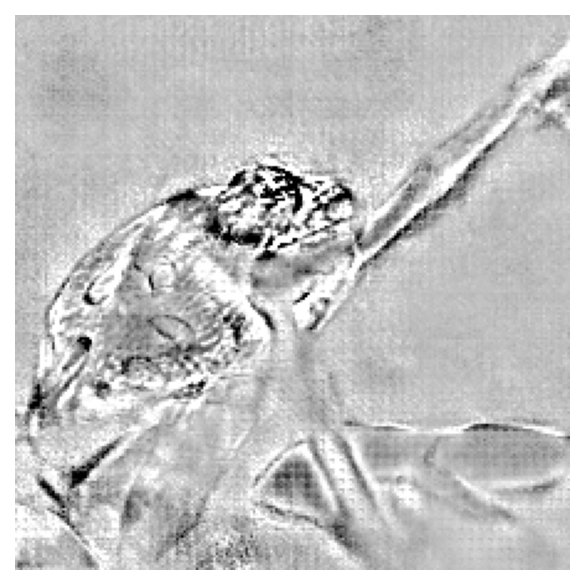}
     \end{subfigure}
    \caption{
    The local centroids of the adversarially Swin-B transformer from~\citet{rodriguez-munozCharacterizingModelRobustness2024}.}
    \label{fig:saliency_map3}
\end{figure}

\begin{figure}[p]
     \centering
     \begin{subfigure}[b]{0.48\textwidth}
         \centering
         \includegraphics[width=0.6\textwidth]{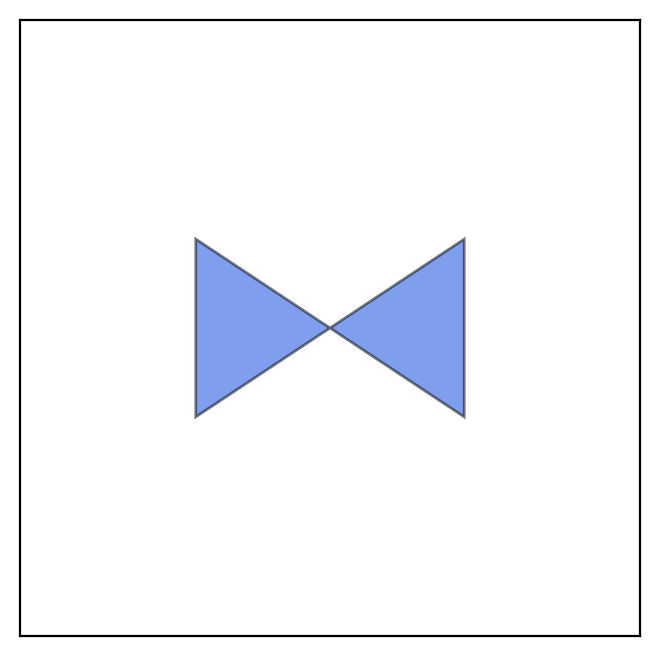}
         \caption*{Input Distribution}
     \end{subfigure}
     \hfill
     \begin{subfigure}[b]{0.48\textwidth}
         \centering
         \includegraphics[width=0.6\textwidth]{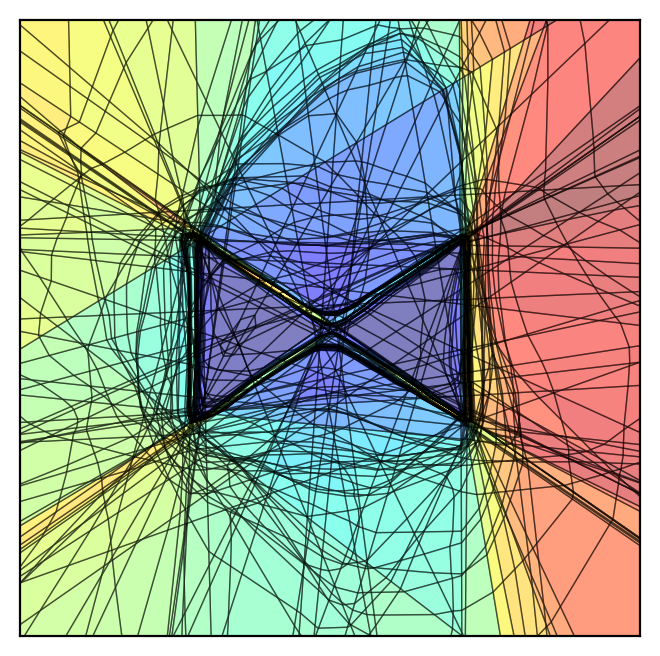}
         \caption*{Functional Geometry}
     \end{subfigure}
     \hfill
     \begin{subfigure}[b]{0.3\textwidth}
         \centering
         \includegraphics[width=\textwidth]{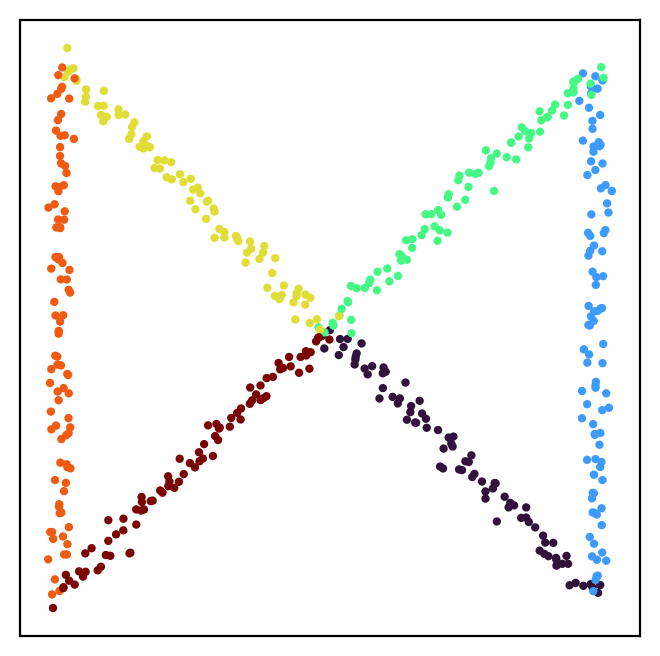}
         \caption*{Input Samples}
     \end{subfigure}
     \hfill
     \begin{subfigure}[b]{0.3\textwidth}
         \centering
         \includegraphics[width=\textwidth]{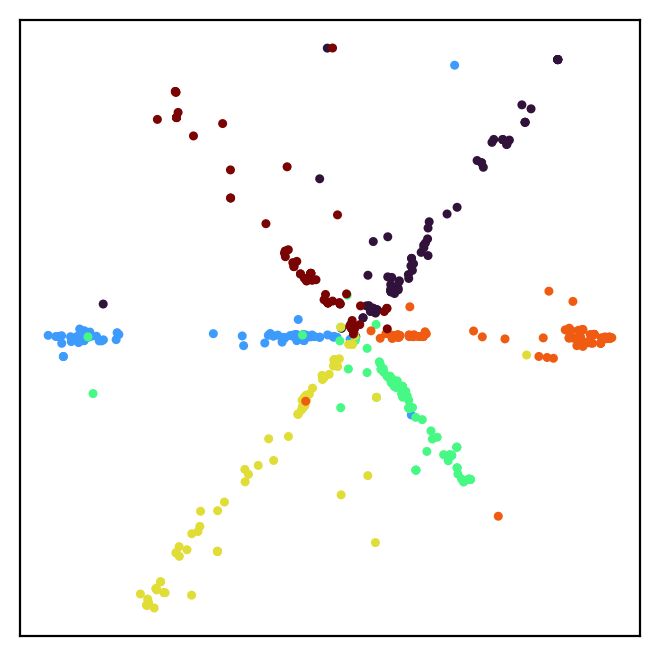}
         \caption*{Centroids}
     \end{subfigure}
     \hfill
     \begin{subfigure}[b]{0.3\textwidth}
         \centering
         \includegraphics[width=\textwidth]{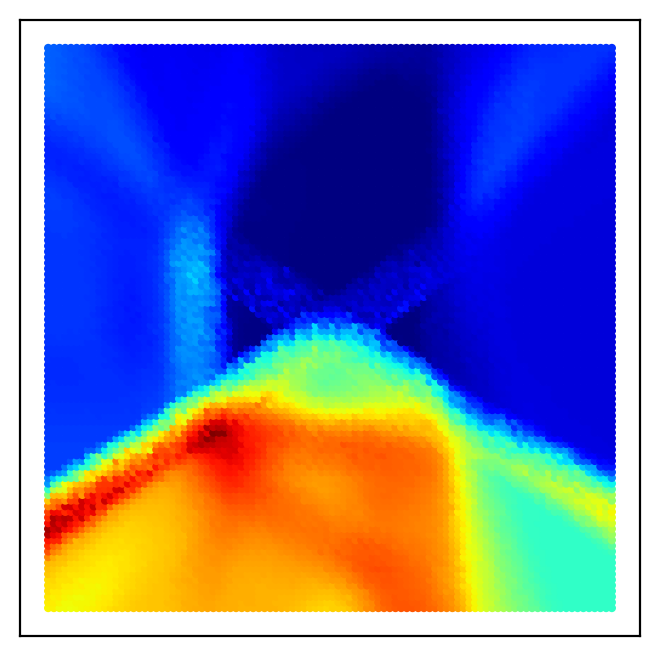}
         \caption*{$s^{(i,3)}$}
     \end{subfigure}
    \hfill
     \begin{subfigure}[b]{0.48\textwidth}
         \centering
         \includegraphics[width=0.6\textwidth]{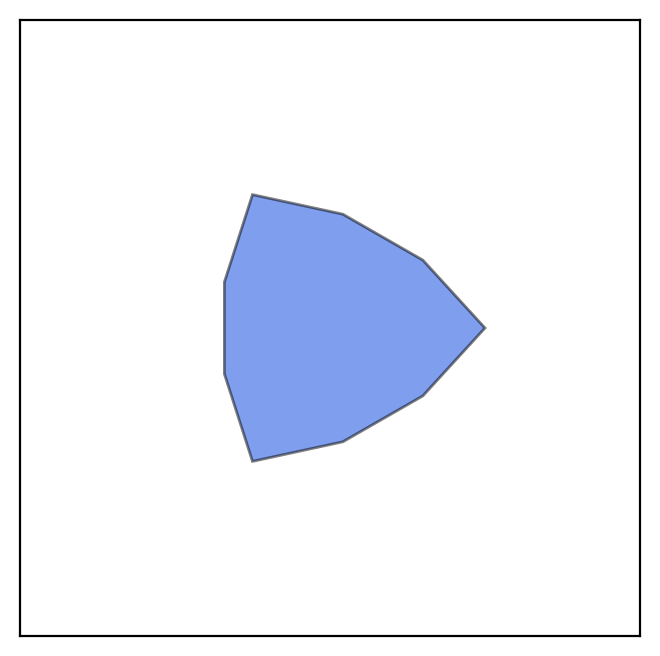}
         \caption*{Input Distribution}
     \end{subfigure}
     \hfill
     \begin{subfigure}[b]{0.48\textwidth}
         \centering
         \includegraphics[width=0.6\textwidth]{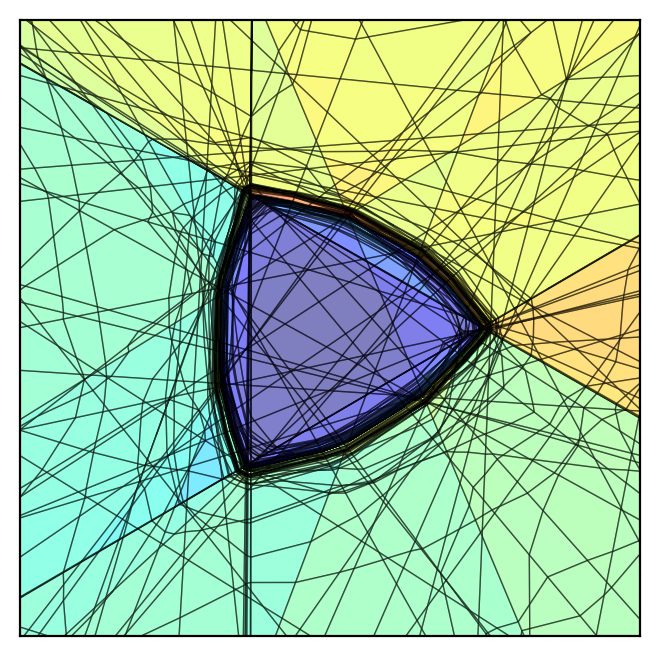}
         \caption*{Functional Geometry}
     \end{subfigure}
     \hfill
     \begin{subfigure}[b]{0.3\textwidth}
         \centering
         \includegraphics[width=\textwidth]{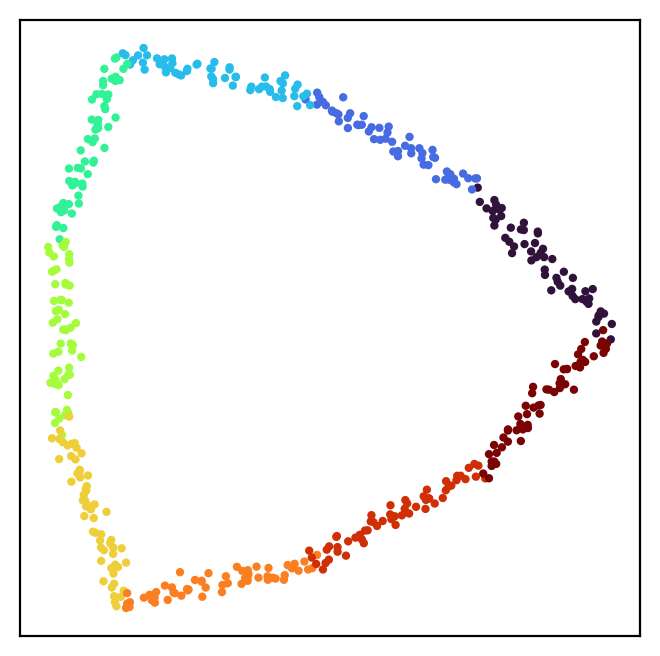}
         \caption*{Input Samples}
     \end{subfigure}
     \hfill
     \begin{subfigure}[b]{0.3\textwidth}
         \centering
         \includegraphics[width=\textwidth]{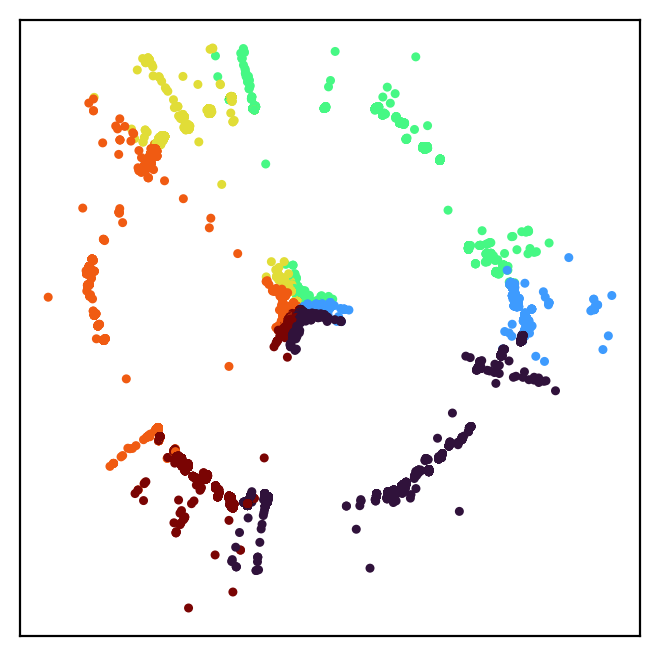}
         \caption*{Centroids}
     \end{subfigure}
     \hfill
     \begin{subfigure}[b]{0.3\textwidth}
         \centering
         \includegraphics[width=\textwidth]{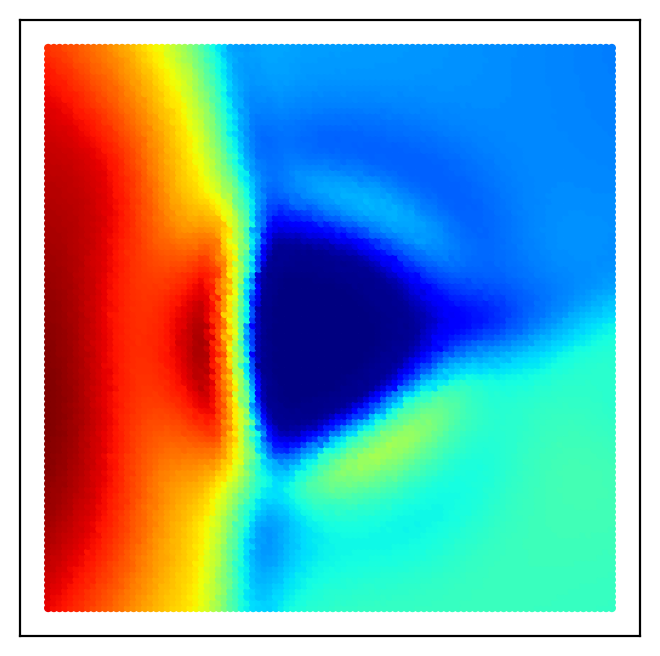}
         \caption*{$s^{(i,2)}$}
     \end{subfigure}
    \caption{
    Here we perform some of the same analyses as conducted previously, but with a DN trained on a bowtie-shaped polygon, \textbf{top} two rows, and a reuleaux-shaped polygon, \textbf{bottom} two rows.}
    \label{fig:other_shapes}
\end{figure}
\begin{figure}[p]
    \centering
    \includegraphics[width=\linewidth]{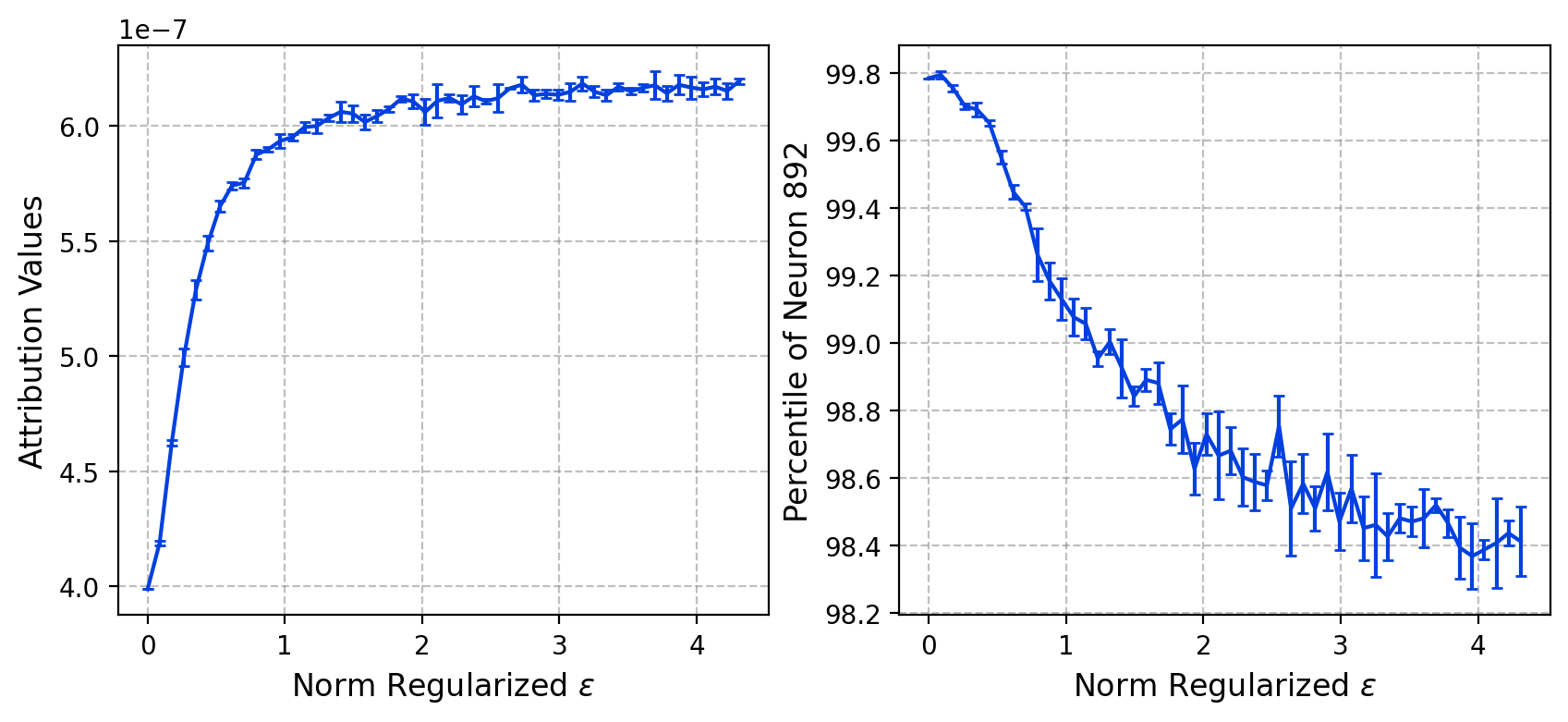}
    \caption{
    The neuron attribution metric of \Cref{eq:neuron_attribution} is a robust measure for interpreting the neurons of a DN.
    Here we consider the robustness of \Cref{eq:neuron_attribution} for the experiment in \Cref{fig:circuit_discovery}.
    More specifically, we test how the neurons' attribution values change as we consider increasingly large neighborhoods.
    The neighborhoods we consider are of the form $\mathcal{B}_{\epsilon}(\vx)$, where $\vx$ is the embedding of the last token of a prompt at the $31^{\text{st}}$ of GPT2-Large.
    We consider $\epsilon$ normalized by the norm of the centroid of $\vx$ at this layer of the DN.
    To compute \Cref{eq:neuron_attribution} at each neuron of the layer, we sample $256$ embeddings from this neighborhood.
    In the left plot, we observe how the average attribution value of each neuron changes across random samplings of this neighborhood. 
    In the right plot, we observe how the percentile value of the $892^{\text{nd}}$ changes across these random samplings.
    The error bars represent one standard deviation in the observed values.}
    \label{fig:neuron_attribution_robustness}
\end{figure}

\end{document}